\documentclass{article}

\usepackage{microtype}
\usepackage{graphicx}
\usepackage{subfigure}
\usepackage{booktabs} 

\usepackage{hyperref}
\hypersetup{
   colorlinks=true,
   citecolor=cyan
}

\usepackage[table]{xcolor}
\usepackage[accepted]{icml2023}

\usepackage[export]{adjustbox}

\usepackage{amsmath}
\usepackage{amssymb}
\usepackage{mathtools}
\usepackage{amsthm}
\usepackage{thm-restate}

\usepackage{tcolorbox}

\theoremstyle{plain}
\theoremstyle{definition}
\newtheorem{theorem}{Theorem}[section]
\newtheorem{proposition}[theorem]{Proposition}
\newtheorem{lemma}[theorem]{Lemma}
\newtheorem{corollary}[theorem]{Corollary}
\theoremstyle{definition}

\newtheorem{remark}[theorem]{Remark}

\usepackage{mathrsfs}
\usepackage{enumitem}
\usepackage{nicefrac}
\usepackage{bbm}
\usepackage[framemethod=TikZ]{mdframed}
\mdfdefinestyle{box}{%
backgroundcolor=black!6,
roundcorner=10pt,
innertopmargin=8pt,
innerleftmargin=8pt}

\setlist[itemize,0]{leftmargin=0.5cm,itemsep=-0.1cm,topsep=-0.2cm}
\setlist[enumerate,0]{leftmargin=0.75cm,itemsep=-0.1cm,topsep=-0.2cm}

\definecolor{lightblue}{HTML}{89CFF0}

\DeclareMathSymbol{\indexminus}{\mathbin}{AMSa}{"39}
\usepackage{accents}
\newcommand{\ubar}[1]{\underaccent{\bar}{#1}}

\usepackage{xspace}

\newcommand{\papertitle}{The Statistical Benefits of Quantile Temporal-Difference Learning for Value Estimation}
\icmltitlerunning{\papertitle}

\begin{document}

\twocolumn[
\icmltitle{The Statistical Benefits \\ of Quantile Temporal-Difference Learning for Value Estimation}



\icmlsetsymbol{equal}{*}

\begin{icmlauthorlist}
\icmlauthor{Mark Rowland}{gdm}
\icmlauthor{Yunhao Tang}{gdm}
\icmlauthor{Clare Lyle}{gdm}
\icmlauthor{R\'emi Munos}{gdm}
\icmlauthor{Marc G. Bellemare}{gdm}
\icmlauthor{Will Dabney}{gdm}
\end{icmlauthorlist}

\icmlaffiliation{gdm}{Google DeepMind}

\icmlcorrespondingauthor{Mark Rowland}{markrowland@deepmind.com}

\icmlkeywords{Reinforcement learning, distributional reinforcement learning, temporal-difference learning}

\vskip 0.3in
]



\printAffiliationsAndNotice{}  

\begin{abstract}
    We study the problem of temporal-difference-based policy evaluation in reinforcement learning. In particular, we analyse the use of a distributional reinforcement learning algorithm, quantile temporal-difference learning (QTD), for this task. 
    We reach the surprising conclusion that even if a practitioner has no interest in the return distribution beyond the mean, QTD (which learns predictions about the full distribution of returns) may offer performance superior to approaches such as classical TD learning, which predict only the mean return, even in the tabular setting. 
\end{abstract}

\section{Introduction}

Distributional approaches to reinforcement learning (RL) aim to learn the full probability distribution over random returns an agent may encounter, rather than just the expectation of the random return \citep{morimura2010nonparametric,morimura2010parametric,bellemare2017distributional,bdr2022}. These methods have seen recent empirical successes in domains such as stratospheric balloon navigation \citep{bellemare2020autonomous}, simulated race car control \citep{wurman2022outracing}, and algorithm discovery \citep{fawzi2022discovering}, as well as forming a core component of many successful agents in common simulated reinforcement learning benchmarks \citep{bellemare13arcade,machado2018revisiting,bellemare2017distributional,dabney2018distributional,dabney2018implicit,yang2019fully,vieillard2020munchausen,nguyen2021distributional}, often improving over agents that estimate only the expected return. Notably, the success of these distributional approaches has typically been observed in combination with deep neural networks, and it is commonly hypothesised that the benefits of the distributional approach stem from its interaction with non-linear function approximators such as deep neural networks \citep{bellemare2017distributional,imani2018improving,dabney2021value,sun2022does}, rather than for statistical reasons.

In this paper, however, we reach a surprising conclusion: 
Even in the tabular setting, there are many scenarios where \emph{quantile temporal-difference learning} \citep[QTD;][]{dabney2018distributional}, a distributional RL algorithm which aims to learn quantiles of the return distribution, can more accurately estimate the expected return than classical temporal-difference learning \citep[TD; ][]{sutton1984temporal,sutton1988learning} which predicts only the expected return.

To complement this core finding, we conduct novel theoretical analysis to establish what kinds of value predictions QTD converges to, and crucially how this depends on the number of quantiles that the algorithm estimates. We also examine how both TD and QTD trade-off between the variance of their updates, and their expected progress towards their asymptotic predictions. We find that when estimating a sufficient number of quantiles, QTD is able to converge to value predictions close to the true value function $V^\pi$, yet with individual updates that are guaranteed to be of bounded magnitude.

These insights lead to several testable hypotheses, which we use to conduct a further empirical study to better characterise domains in which QTD offers superior performance to TD, and vice versa, and find several common trends. In particular, we find that in environments with significant stochasticity, QTD often performs better (and in contrast, in (near-)deterministic environments, TD is clearly preferable), and that estimating a low number of quantiles may have adverse effects on the accuracy of QTD's predictions. By investigating a variant of QTD, we also find evidence that estimation specifically of the return distribution may lead to useful variance-reduction properties.

These findings have consequences for both theoreticians and practitioners. For the former, QTD represents a distinct fundamental approach to the problem of value prediction in RL, often with complementary performance to classical TD, and raises a range of open questions. For the latter, QTD for mean estimation can be considered as a plug-in alternative to TD, in tabular settings and beyond.

\section{Background}

We consider a Markov decision process, specified by a finite state space $\mathcal{X}$, action space $\mathcal{A}$, joint transition probabilities $P : \mathcal{X} \times \mathcal{A} \rightarrow \mathscr{P}(\mathcal{X} \times \mathbb{R})$ that specify for each $(x,a) \in \mathcal{X} \times \mathcal{A}$ a distribution over an immediate reward and next state,
and a discount factor $\gamma \in [0, 1)$.
When a policy $\pi : \mathcal{X} \rightarrow \mathscr{P}(\mathcal{A})$ for selecting actions at each state is specified, each choice of an initial state $x_0 \in \mathcal{X}$ gives rise to a probability distribution over the trajectory $(X_t, R_t)_{t \geq 0}$, which we refer to as a Markov reward process (MRP). We introduce the notation $\mathbb{P}^\pi_{x_0}$ for the distribution of the trajectory, and $\mathbb{E}^\pi_{x_0}$ for the corresponding expectation operator.

The discounted return (or simply, the return) obtained along the trajectory $(X_t, R_t)_{t \geq 0}$ is defined as
\begin{align}\label{eq:random-return}
    \textstyle
    \sum_{t \geq 0}  \gamma^t R_t \, ,
\end{align}
and encodes the utility of the trajectory to an agent interacting with the environment; higher returns are better.
The value function $V^\pi : \mathcal{X} \rightarrow \mathbb{R}$ for a policy $\pi$ is defined by
\begin{align*}
    \textstyle
    V^\pi(x) = \mathbb{E}^\pi_x\left[ \sum_{t \geq 0} \gamma^t R_t \right] \, ,
\end{align*}
for each $x \in \mathcal{X}$. That is, $V^\pi(x)$ is the mean return encountered along trajectories beginning at the state $x$. Estimating $V^\pi$ from observed interactions with the environments is a fundamental problem in reinforcement learning, as it allows agents to both predict the effects of their actions, and to improve their policies.

\subsection{Temporal-difference learning}\label{sec:td-background}

Temporal-difference learning \citep[TD;][]{sutton1984temporal,sutton1988learning} is a family of algorithms that aim to learn an estimate of the value function $V^\pi$. We focus here on the simplest variant, TD(0). This algorithm maintains an estimate $V$ of the value function, and incrementally updates these predictions in response to experience in the environment. 
Specifically, on observing a transition $(x, r, x')$ generated by $\pi$, TD learning selects a learning rate $\alpha\in (0, 1]$, and performs the assignment
\begin{align}\label{eq:td-update}
    V(x) \leftarrow V(x) + \alpha (r + \gamma V(x') - V(x)) \, ,
\end{align}
to update $V$. Under appropriate conditions, the estimate $V$ converges to the true value function $V^\pi$ with probability 1. The TD(0) update rule is a central method for tabular policy evaluation, and moreover, this update and its variants forms a core component of many deep reinforcement learning agents, including value-based approaches such as DQN and its descendants \citep{mnih15human}, as well as actor-critic approaches such as A3C and its descendants \citep{mnih2016asynchronous}. For conciseness, throughout the paper we use ``TD'' to refer to this classical temporal-difference learning algorithm.

\subsection{Quantile temporal-difference learning}

In contrast to TD learning, which maintains an estimate of the expected return at each state, distributional RL algorithms \citep{bdr2022} aim to predict the full probability distribution of the random return in Equation~\eqref{eq:random-return}. Quantile temporal-difference learning \citep[QTD][]{dabney2018distributional}, in particular, maintains a \emph{collection} of predictions at each state, denoted $((\theta(x, i))_{i=1}^m : x \in \mathcal{X})$, and has formed a core component of many deep reinforcement learning agents \citep{dabney2018implicit,yang2019fully,bodnar2020quantile,bellemare2020autonomous,wurman2022outracing,fawzi2022discovering}.
The intention, in contrast to having $V(x)$ directly approximate the mean return from $x$ in TD learning, is to have $\theta(x, i)$ approximate the $\tau_i$-quantile of the distribution of the random return in Equation~\eqref{eq:random-return}, with $\tau_i = \tfrac{2i-1}{2m}$, for $i=1,\ldots,m$; see Figure~\ref{fig:example-qtd} for an illustration. We write QTD($m$) to denote the instantiation of QTD with $m$ quantiles.

In analogy with TD learning, upon observing a transition $(x, r, x')$  generated by $\pi$, QTD updates all estimates $(\theta(x, i))_{i=1}^m$ at state $x$ by selecting a learning rate $\alpha \geq 0$, and performing the assignments
\begin{align}
    \textstyle
    & \theta(x, i) \leftarrow  \theta(x,i) + \label{eq:qtd}\\
    & \qquad \alpha \Big( \tau_i - \frac{1}{m} \sum_{j=1}^m \mathbbm{1}[ r + \gamma \theta(x', j) - \theta(x, i) < 0 ] \Big) \, , \nonumber
\end{align}
for all $i=1,\ldots,m$. This algorithm differs from TD in a few important ways. First, each prediction $\theta(x, i)$ is updated differently, due to the presence of the parameter $\tau_i$ in the update. Second, the update depends only on the \emph{sign} (not magnitude) of the temporal-difference errors appearing in Equation~\eqref{eq:qtd}, meaning that the update magnitude is bounded, in contrast to those of TD.
The form of the update itself is motivated through the quantile regression loss \citep{koenker1978regression,koenker2005quantile}; see \citet{rowland2023analysis} and \citet{bdr2022} for further background and theory regarding QTD, and Appendix~\ref{appendix:comp} for an overview and discussion of computational considerations.

\begin{figure}[b]
    \centering
    \ \ \ \ \ \includegraphics[width=.43\textwidth]{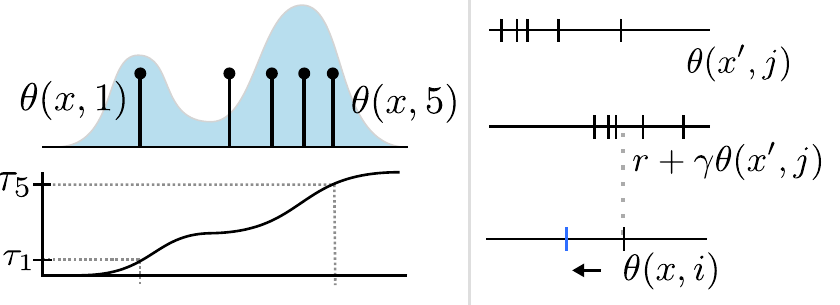}
    \caption{Left: Example return distribution at a state $x$ in light blue, with  QTD quantile predictions $(\theta(x, i))_{i=1}^5$ ($m=5$). Exact predictions correspond to the quantiles indicated on the CDF below. Right: Illustration of computation of the update to $\theta(x, i)$ (from black to blue marker) upon observing a transition $(x, r, x')$.}
    \label{fig:example-qtd}
\end{figure}

\section{Quantile temporal-difference learning for mean-return estimation}

Quantile temporal-difference learning estimates the return distribution at a state $x$ with the discrete distribution supported at the learnt quantile values $(\theta(x,i))_{i=1}^m$:
\begin{align*}
    \sum_{i=1}^m \frac{1}{m} \delta_{\theta(x,i)} \, ;
\end{align*}
see Figure~\ref{fig:example-qtd}. A natural estimator for value at state $x$ is therefore obtained by extracting the mean of this approximate distribution, by averaging these quantiles:
\begin{align}\label{eq:qtd-mean-estimate}
    \frac{1}{m} \sum_{i=1}^m \theta(x, i) \, .
\end{align}
This is the estimator of value typically used in applications combining QTD with deep reinforcement learning. The approach of averaging certain quantile estimators to approximate the mean of a distribution in fact dates back to at least the work of \citet{daniell1920observations} and \citet{mosteller1946some}, with \citet{gastwirth1966robust} observing that this approach to estimation provides competitive relative efficiency across a wide variety of distributions, including those with heavy tails, where the usual sample-average mean estimator can be inefficient.

Thus, although not originally designed with this connection in mind, QTD naturally combines this approach to mean estimation with the notion of bootstrapping in reinforcement learning. Given the motivation above, we might conjecture that QTD provides an approach to value estimation that is effective across a wide range of environments, particularly those with heavy-tailed reward distributions. For concreteness, the QTD algorithm for value estimation is presented in Algorithm~\ref{alg:qtd-mean}. The additional variables $\theta'$ are used to avoid issues when $x'_t = x_t$; in such cases, this means that the for-loop over the quantile index $i$ can be performed in any order (or in parallel) without affecting the result of the algorithm.

\begin{algorithm}
    \begin{algorithmic}[1]
        \REQUIRE Initial quantile estimates $((\theta(x, i))_{i=1}^m : x \in \mathcal{X})$, learning rate $\alpha$, number of updates $T$.
        \FOR{$t=1,\ldots,T$}
            \STATE Observe transition $(x_t, r_t, x'_t)$.
            \FOR{$i=1,\ldots,m$}
                \STATE Set $ \theta'(x_t, i) \leftarrow \theta(x_t,i) + $ \\
                    $\  \alpha \Big( \tau_i - \frac{1}{m} \sum_{j=1}^m \mathbbm{1}[\theta(x_t, i) - r_t - \gamma \theta(x'_t, j) < 0 ] \Big) $ \label{algline:update}
            \ENDFOR
            \STATE Set $\theta(x_t, i) \leftarrow \theta'(x_t, i)$ for $i=1,\ldots,m$
        \ENDFOR
        \STATE \textbf{return} $(\tfrac{1}{m} \sum_{i=1}^m \theta(x, i) : x \in \mathcal{X})$
    \end{algorithmic}
    \caption{QTD($m$) for value estimation.}
    \label{alg:qtd-mean}
\end{algorithm}

As an initial comparison between TD and QTD for value estimation, we compare mean-squared error for value estimation on a suite of nine simple MRPs. Full details for replication are provided in Appendix~\ref{sec:further-experiment-details}, with crucial details for the comparisons given here. The structure of the MRPs is given by the Cartesian product of three levels of stochasticity in transition structure:
\begin{itemize}
    \item Deterministic cycle structure;
    \item Sparse stochastic transition structure (sampled from a Garnet distribution; 
    \item Dense stochastic transition structure (sampled from Dirichlet($1,\ldots,1)$ distributions); \citeauthor{archibald1995generation}, \citeyear{archibald1995generation});
\end{itemize}
together with three levels of stochasticity in reward structure:
\begin{itemize}
    \item Deterministic rewards;
    \item Gaussian (variance 1) rewards;
    \item Exponentially distributed (rate 1) rewards.
\end{itemize}

We focus on the use of constant learning rates throughout training, as is commonly the case in practice, and sweep across a variety of learning rates for both methods. We run both TD and QTD (using 128 quantiles) with a variety of learning rates, and measure the mean-squared error to the true value function after 1,000 updates via online interaction with the environments. The results of the sweep over learning rates are displayed in Figure~\ref{fig:main-sweep}; in this experiment and all that follow, each run was repeated 1,000 times, and the (narrow) confidence bands displayed are obtained via a measurement of $\pm$2 times the empirical standard error.

\begin{figure}[h]
    \centering
    
    \includegraphics[width=.48\textwidth]{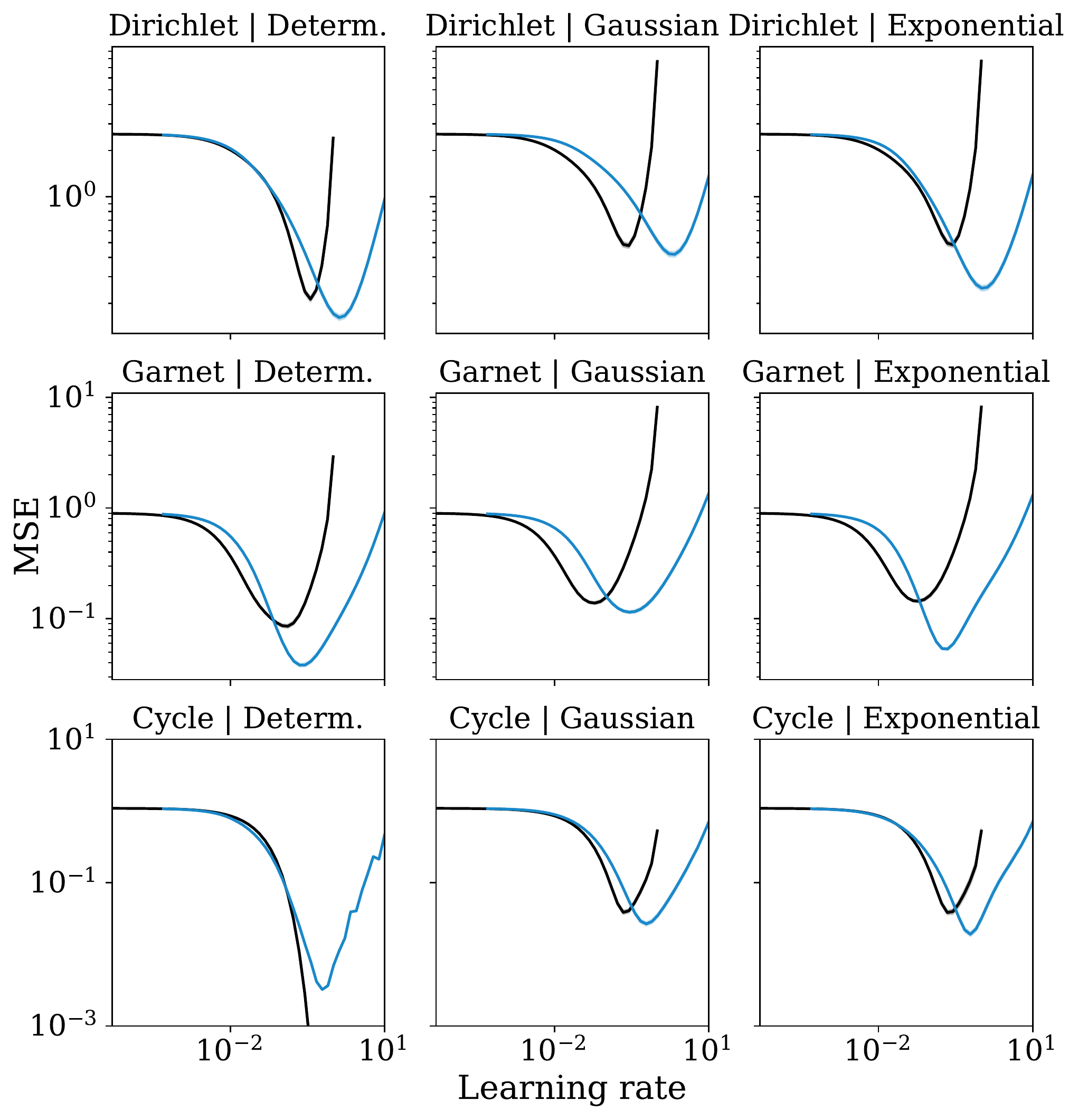}
    
    \caption{Mean-squared error against learning rate for TD (black) and QTD(128) (blue), on environments with Dirichlet transition structure (top), Garnet transition structure (middle), and deterministic cycle structure (bottom), and deterministic rewards (left), Gaussian rewards (centre), and exponential rewards (right).}
    \label{fig:main-sweep}
\end{figure}

As expected, in the environments with the heaviest-tailed rewards, QTD obtains a lower mean-squared error than TD. Interestingly, this is also the case in environments with stochasticity only in the transition dynamics, and deterministic rewards. To more easily visualise the extent of these improvements, and to check the robustness of this improvement to the number of updates undertaken by the algorithms, we plot the optimal MSE obtained by QTD as a proportion of that obtained by TD in Figure~\ref{fig:main-sweep-over-updates}, as a function of the number updates completed by each algorithm. This preliminary experiment has already yielded a perhaps surprising conclusion:

\begin{tcolorbox}
    \begin{center}
        \emph{
        Even in the tabular setting, QTD, a distributional
        reinforcement learning algorithm, can outperform
        classical TD learning in estimating expected returns.
        }
    \end{center}
\end{tcolorbox}

In addition to obtaining superior performance relative to TD when optimising over learning rates, Figure~\ref{fig:main-sweep} also indicates that performance degradation due to a larger-than-optimal learning rate is considerably less severe with QTD than with TD in these environments.

Importantly, however, we also note that for the deterministic environment in the suite, the performance of TD is far superior to QTD. In this case, the TD algorithm is able to very accurately approximate the value function, since the update in Equation~\ref{eq:td-update} is essentially implementing exact asynchronous dynamic programming when $\alpha=1$.

The results above have shown that in some sense, QTD has a complementary performance profile to TD, viewed as algorithms for value estimation, and that the stochaticity of the environment is one important factor in determining the relative performance of QTD and TD.
What else can be said about the performance of QTD in comparison with TD? 
We address this questions in two ways. First, we develop the theory of QTD for mean estimation in Section~\ref{sec:theory}, establishing asymptotic guarantees on the quality of the value predictions that the algorithm makes. Second, we conduct further empirical investigations in Section~\ref{sec:further-experiments}, aiming to develop a more nuanced understanding of the relative performance of QTD and TD in practice.

\begin{figure}
    \centering
    
    \includegraphics[width=.48\textwidth]{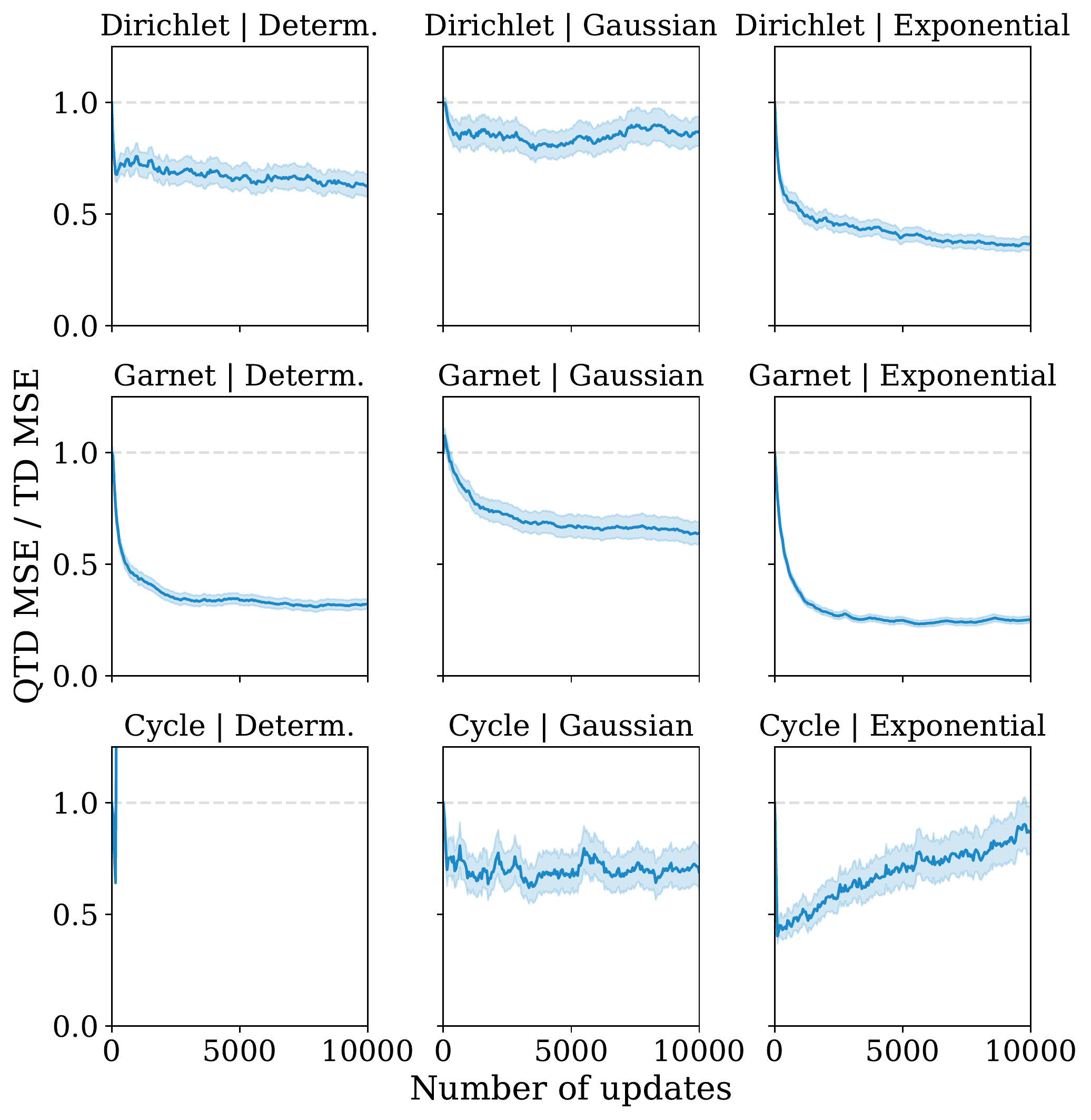}
    
    \caption{Improvement of QTD(128) over TD in mean-squared error against number of updates for all environments in Figure~\ref{fig:main-sweep}.}
    \label{fig:main-sweep-over-updates}
\end{figure}

\section{Theoretical analysis}\label{sec:theory}

For the TD learning update rule in Section~\ref{sec:td-background}, it is known that under mild conditions on the reward distributions of the MRP, learning rates, and frequency that each state $x$ is updated, the predictions $V$ converge to $V^\pi$ with probability 1 \citep{watkins1989learning,watkins1992q,dayan1992convergence,dayan1994td,tsitsiklis1994asynchronous,jaakkola1994convergence,bertsekas1996neuro}. This core convergence result justifies the use of TD for  policy evaluation.

\renewcommand*{\thefootnote}{\fnsymbol{footnote}}

\citet{rowland2023analysis} show that under even milder conditions, QTD($m$) also converges with probability 1
\footnote{QTD may converge to a \emph{set} of fixed points, rather than a unique fixed point as with TD; the discussion above refers to any point in this set of fixed points.}.
However, in general, the estimate of expected returns $V^{\text{QTD}}_m$ extracted from a point of convergence $\theta^{\text{QTD}}_m$ of QTD via Equation~\eqref{eq:qtd-mean-estimate} is \emph{not} exactly equal to $V^\pi$.
Intuitively, this stems from the fact that the algorithm is not estimating the mean of the return distribution directly, but rather a finite collection of quantiles, and this information is insufficient to exactly reconstruct the mean of the distribution in question.

Based on this observation, we bound the expected error of the (random) value estimator $\hat{V}$ obtained from  Algorithm~\ref{alg:qtd-mean} as follows:
\begin{align}
    \mathbb{E}[\|\hat{V} - V^\pi \|] \leq 
    \underbrace{\mathbb{E}[\|\hat{V} - V^{\text{QTD}}_m \|]}_{\text{Finite-sample error}} + \underbrace{\| V^{\text{QTD}}_m - V^\pi \|}_{\text{Fixed-point error}} \, . \label{eq:decomposition}
\end{align}
The precise norm is unimportant here; the main aim is to highlight the role played by fixed-point error and finite-sample error in the overall error incurred by the QTD value estimator. We now compare QTD and TD with each of these terms in mind, beginning with fixed-point error.

\subsection{Fixed-point error}

As noted above, TD incurs zero fixed-point error, as its point of convergence is precisely $V^\pi$. However, this is generally not true of QTD. Nevertheless, it is possible to bound the fixed-point error of QTD as a function of the number of quantiles estimated by QTD in many cases. The following result is a straightforward consequence of the fixed-point analysis of \citet{rowland2023analysis}. Proofs of results stated in the main paper are provided in Appendix~\ref{sec:proofs}.

\begin{restatable}{proposition}{propQTDFixedPoint}\label{prop:qtd-fixed-point-bound}
    For an MRP with all reward distributions supported on $[R_{\text{min}}, R_{\text{max}}]$, any convergence point $\theta^{\text{QTD}}_m$ of QTD($m$) with corresponding value function estimate $V^{\text{QTD}}_m = (\tfrac{1}{m} \sum_{i=1}^m \theta^{\text{QTD}}_m(x, i) : x \in \mathcal{X})$ satisfies
    \begin{align*}
        \| V^{\text{QTD}}_m - V^\pi \|_\infty \leq \frac{R_{\text{max}} - R_{\text{min}}}{2m(1-\gamma)^2} \, .
    \end{align*}
\end{restatable}

This guarantees that in the case of bounded reward distributions, we can ensure that the fixed points of QTD provide arbitrarily accurate value function estimates, as long as $m$ is taken to be sufficiently large relative to the scale of the support of the reward distributions.

\begin{remark}\label{remark:intuition}
    The form of this approximation error is easily interpreted; for a general distribution supported on $[R_\text{min}/(1-\gamma), R_\text{max}/(1-\gamma)]$ (as the return distribution at $x$ is under the conditions of Proposition~\ref{prop:qtd-fixed-point-bound}), with mean $\mu$ and quantile function $F^{-1}$, we have
    \begin{align}\label{eq:integrand}
        \mu = \int_0^1 F^{-1}(\tau)\; \mathrm{d} \tau \approx \sum_{i=1}^m \frac{1}{m} F^{-1}\Big(\frac{2i-1}{2m}\Big) \, .
    \end{align}
    That is, estimating the mean with a finite number of quantiles can be understood as a midpoint-quadrature-based approximation to the true mean. From this point of view, a linear dependence of the error on the range $(R_\text{max} - R_\text{min})/(1-\gamma)$ of the integrand in Equation~\eqref{eq:integrand}, and a dependence $1/m$ on the number of quadrature points $m$ are to be expected. The additional factor of $(1-\gamma)^{-1}$ in the bound stems from the fact that the estimate is obtained from the fixed point of a bootstrapping procedure, in which errors accumulate at each stage. Since only a finite number $m$ of quantiles are estimated at each state, the remaining information about the return distributions is thrown away, and this results in an accumulation of error each time the update in Equation~\eqref{eq:qtd} is applied. This bears a relationship to the notion of Bellman closedness in distributional RL \citep{rowland2019statistics,bdr2022}, and is analogous to the compounding of error under linear function approximation \citep{tsitsiklis1997analysis}.
\end{remark}

We now develop this analysis further, obtaining results for environments with unbounded rewards. We state a bound for the important case of sub-Gaussian rewards below, which follows as a consequence of a much more general bound given by Proposition~\ref{prop:abstract-fixed-point-quality}.

\begin{restatable}{proposition}{propQTDSubGaussian}\label{prop:qtd-fixed-subgaussian}
    Consider an MRP with all reward distributions having means in $[R_{\text{min}}, R_{\text{max}}]$, and all sub-Gaussian with parameter $\sigma^2$, so that $\mathbb{E}^\pi_x[\exp(\lambda (R-\mathbb{E}^\pi_x[R]))] \leq \exp(\lambda^2 \sigma^2/2)$, for all $\lambda \in \mathbb{R}$ and $x \in \mathcal{X}$. Then for the value function estimate $V^{\text{QTD}}_m$ obtained from any convergence point $\theta^{\text{QTD}}_m$ of QTD($m$) via Equation~\eqref{eq:qtd-mean-estimate}, we have
    \begin{align*}
        &\| V^{\text{QTD}}_m - V^\pi \|_\infty \leq \frac{1}{(1-\gamma)m} \times \\
        &\left( \frac{R_\text{max} - R_\text{min} + 2 \sigma \sqrt{2 \log (2m)}}{2(1-\gamma)} + \frac{\sigma}{\sqrt{2 \log(2m)}}  \right) \, .
    \end{align*}
\end{restatable}

We also state a non-quantitative result applicable to any MDP for which the problem of mean return estimation is well defined.

\begin{restatable}{proposition}{propQTDFiniteMean}\label{prop:qtd-finite-mean}
    Consider an MDP with all reward distributions having finite mean. Then for the value function estimate $V^{\text{QTD}}_m$ obtained from any convergence point $\theta^{\text{QTD}}_m$ of QTD($m$) via Equation~\eqref{eq:qtd-mean-estimate}, we have $\| V^{\text{QTD}}_m - V^\pi \|_\infty \rightarrow 0$ as $m \rightarrow \infty$.
\end{restatable}

This analysis shows that even with unbounded reward distributions, the approximation error of the fixed points of QTD can still be made arbitrarily small by increasing $m$, with a slightly slower rate (relative to the bounded-reward case) of $O(m^{-1}\sqrt{\log(m)}(1-\gamma)^{-2})$ in the case of sub-Gaussian rewards; in general, the heavier the tails of the reward distributions, the slower the convergence may be.

\subsection{Expected updates and variance}

The analysis of the previous section shows that QTD (with a large enough number of quantiles) incurs low fixed-point error, but does not suggest how its finite-sample performance may compare to that of TD, and specifically in which kinds of environments it may outperform TD. To make progress on this question, we return to the other term in Inequality~\eqref{eq:decomposition}, and in particular consider how the updates of TD and QTD contribute to this quantity. We begin by considering the updates of TD.

\textbf{TD update decomposition.} The right-hand side of the TD learning update in Equation~\eqref{eq:td-update} can be rewritten as
\begin{align*}
    & V(x) + \alpha(r + \gamma V(x') - V(x)) \\
    = & (1-\alpha) V(x) + \\
    &\quad \alpha (\underbrace{(T^\pi V)(x)}_{\substack{\text{Expected} \\ \text{update }}}
    + \underbrace{(r + \gamma V(x') - (T^\pi V)(x))}_{\text{Mean-zero noise}}) \, ,
\end{align*}
where $(T^\pi V)(x) = \mathbb{E}^\pi_x[R_0 + \gamma V(X_1)]$ is the classical dynamic programming operator. This decomposition is central to the analyses of TD cited above, and highlights that the learning rate $\alpha$ balances two requirements: a large learning rate increases the expected update towards $(T^\pi V)(x)$, increasing the contraction towards the fixed point $V^\pi$ of $T^\pi$, but also amplifies the mean-zero noise. Note also that the magnitude of the noise is potentially unbounded (if there are unbounded rewards, or if the value estimate $V$ grows large), and that the distance of the expected update $(T^\pi V)(x)$ also grows in magnitude with $V$.

The key to obtaining good performance, and low finite-update error, from TD is therefore selecting a learning rate that balances the tension between these two considerations. These links between temporal-difference learning and dynamic programming are well understood (see e.g.\ \citet{jaakkola1994convergence,tsitsiklis1994asynchronous,bertsekas1996neuro}), and this specific trade-off has been previously quantified under a variety of formalisms; see the work of \citet{kearns2000bias} for the \emph{phased} setting, and \citet{even2003learning} for the synchronous and online settings. 

\begin{table}[t]
    \centering
    \begin{tabular}{c|c|c|c}
           & \shortstack{Fixed-point \\ bias} & \shortstack{Update \\ variance} & \shortstack{Expected \\ update magnitude} \\ \hline
        TD & 0 & Unbounded{*} & $\propto \text{Bellman error}$ \\
        QTD & $\widetilde{\mathcal{O}}(1/m)$\textsuperscript{**} & $\mathcal{O}(1)$ & $\mathcal{O}(1)$
    \end{tabular}
    \caption{Trade-offs made by TD and QTD along various axes. \textsuperscript{*}In general, TD update variance may be unbounded, though there are certain situations in which it is not; see text for further discussion. \textsuperscript{**}For sub-Gaussian reward distributions. $\tilde{O}$ denotes the possible dropping of polylog factors in $m$.}
    \label{tab:comparison}
\end{table}

\textbf{QTD update decomposition.} In analogy, we can also decompose the QTD update in Equation~\eqref{eq:qtd} into an expected update and mean-zero noise; this approach is central to the convergence analysis of \citet{rowland2023analysis}. In particular, the right-hand side of Equation~\eqref{eq:qtd} can be decomposed as follows:
\begin{align*}
     & \alpha \Big( \tau_i - \frac{1}{m} \sum_{j=1}^m \mathbbm{1}[ r + \gamma \theta(x', j)   < \theta(x, i) ] \Big) \\
     = & \alpha \Big( \underbrace{\tau_i - \mathbb{P}^\pi_x(\Delta_{iJ}(x, R_0, X_1) < 0)}_{\text{Expected update}} +  \\
     &  \underbrace{\mathbb{P}^\pi_x(\Delta_{iJ}(x,R, X') < 0) - \frac{1}{m} \sum_{j=1}^m \mathbbm{1}[ \Delta_{ij}(x, r, x') < 0]}_{\text{Mean-zero noise}} \Big)\, .
\end{align*}
where we write $\Delta_{ij}(x, r, x') = r + \gamma \theta(x', j) -\theta(x, i)$. 
\citet{rowland2023analysis} show that following these expected updates leads to the points of convergence for QTD, and we therefore have a similar tension as described in TD, between an expected update that moves us towards the points of convergence, and noise that may perturb this progress.

A central distinction between this decomposition for TD, and for QTD, is that in QTD both expected update and noise are bounded by 1, in stark contrast to the potentially unbounded terms in the TD update, which may grow in proportion with the value function norm $\|V\|_\infty$. This suggests that QTD may tolerate higher step sizes than TD in stochastic environments, and also that as the level of stochasticity increases, due to higher-variance/heavier-tailed rewards, the performance of QTD may be more resilient than that of TD. Conversely, in near-deterministic environments, since the expected update magnitude of QTD is effectively independent of the magnitude of the update error, we may expect poorer performance than TD, which is able to make updates in proportion to the level of error. A summary of the comparison points highlighted between TD and QTD in this section is given in Table~\ref{tab:comparison}.

\section{Further empirical analysis}\label{sec:further-experiments}

The theoretical analysis in the previous section has elucidated several salient differences between TD and QTD as policy evaluation algorithms; we now seek to compare these methods empirically, and test the predictions made in light of the analysis in the earlier sections.

\subsection{Heavy-tailed rewards}

As alluded to above, the sensitivity of TD updates to the magnitude of prediction errors makes it difficult to average out heavy-tailed noise, and we hypothesise that in cases of extremely heavy-tailed noise, QTD should strongly outperform TD. To this end, we extend the example environments from Figure~\ref{fig:main-sweep}, with $t_2$-distributed rewards; these are exceptionally heavy-tailed rewards, with infinite variance. The results of QTD and TD in these environments are displayed in Figure~\ref{fig:t}, with QTD providing substantial improvements in MSE. A plot of MSE against learning rates is provided in Appendix~\ref{sec:app:t}.

\begin{figure}[h]
    \centering
    
    \includegraphics[keepaspectratio,width=.48\textwidth]{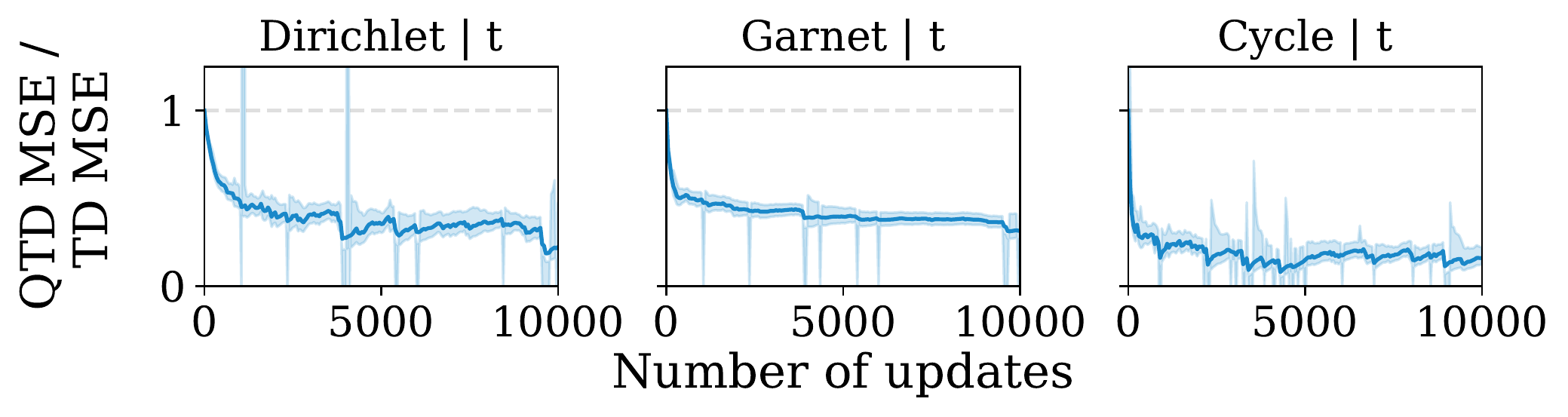}

    \caption{Relative improvement of QTD(128) over TD in mean-squared error against number of updates for all transition structures in Figure~\ref{fig:main-sweep}, with $t_2$-distributed rewards.}
    \label{fig:t}
\end{figure}

\subsection{Low numbers of quantiles}\label{sec:low-m}

Propositions~\ref{prop:qtd-fixed-point-bound}, \ref{prop:qtd-fixed-subgaussian}, and \ref{prop:qtd-finite-mean} suggest that for low numbers of quantiles $m$, the fixed-point bias of QTD may dominate the error decomposition described above, meaning that it may be outperformed by TD in certain environments. Figure~\ref{fig:low-m} illustrates such a case, under the same experimental set-up as earlier in the section; MSE is poor with low values of $m$ for all learning rates, and comparable performance to TD is recovered by increasing $m$.

\begin{figure}[h]
    \centering
    \null
    \hfill
    \includegraphics[keepaspectratio,width=.145\textwidth,valign=c]{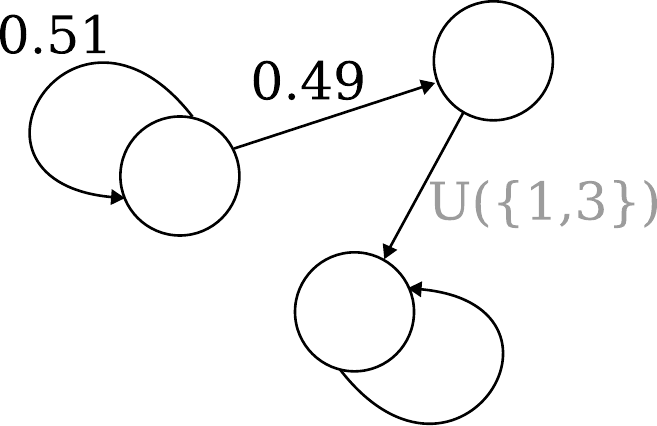}
    \hfill
    \includegraphics[keepaspectratio,width=.325\textwidth,valign=c]{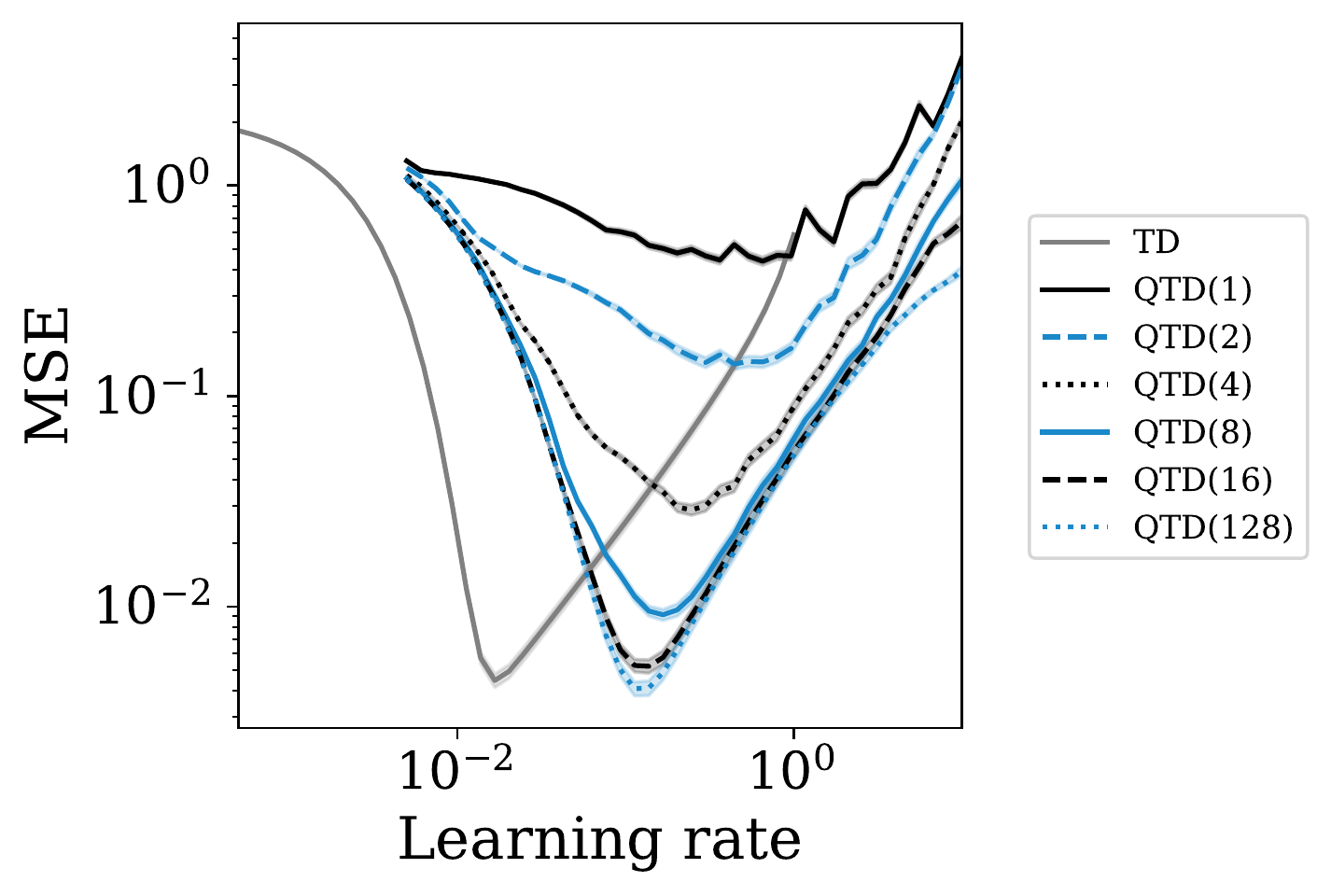}
    \hfill
    \null
    \caption{An example environment (left) where low values of $m$ induce particularly high fixed-point bias. learning rate vs. mean-squared error for QTD($m$) with varying $m$ (right).}
    \label{fig:low-m}
\end{figure}

What makes QTD($m$) with $m$ low behave so poorly in this example? It is precisely due to the fact that the average of a low number of quantiles in this domain is quite different from the mean, as mentioned in Remark~\ref{remark:intuition}. This serves to illustrate cases where large numbers of quantiles are necessary for accurate predictions, as the theory in Section~\ref{sec:theory} suggests. We also include results on the main suite of environments for QTD(1) and QTD(16) in Appendix~\ref{sec:app:main-sweep-m}, for comparison with the results obtained for QTD(128) above. QTD(1) is outperformed by TD in several environments, as the experiment in Figure~\ref{fig:low-m} suggests may be the case. On the other hand, the performance of QTD(16) is broadly in line with that of QTD(128); speaking pragmatically, we have found that using on the order of tens of quantiles is generally sufficient in practice.

\subsection{Varying reward scales}

Given our previous observations that TD outperforms QTD in deterministic environments, and that QTD tends to outperform TD in environments with significant stochasticity, we run an additional comparison to investigate the levels of stochasticity required to see benefits from QTD. In Figure~\ref{fig:main-sweep}, we see advantages to QTD in all environments with stochastic transition structure, but a clear difference in performance in passing from the environment with cycle transition structure and Gaussian reward (centre-bottom) to the same transition structure with deterministic rewards (bottom-left). In Figure~\ref{fig:reward-scale-sweep}, we plot the relative performance of QTD(128) and TD in the environment with cycle transition structure, and Gaussian rewards of varying levels of standard deviation. The results show that at low levels of reward noise, the performance of TD is far superior to QTD, as in the purely deterministic case, with the relative performance of QTD improving monotonically as a function of the standard deviation of the reward noise.

\begin{figure}[h]
    \centering
    \includegraphics[keepaspectratio,width=.4\textwidth,valign=c]{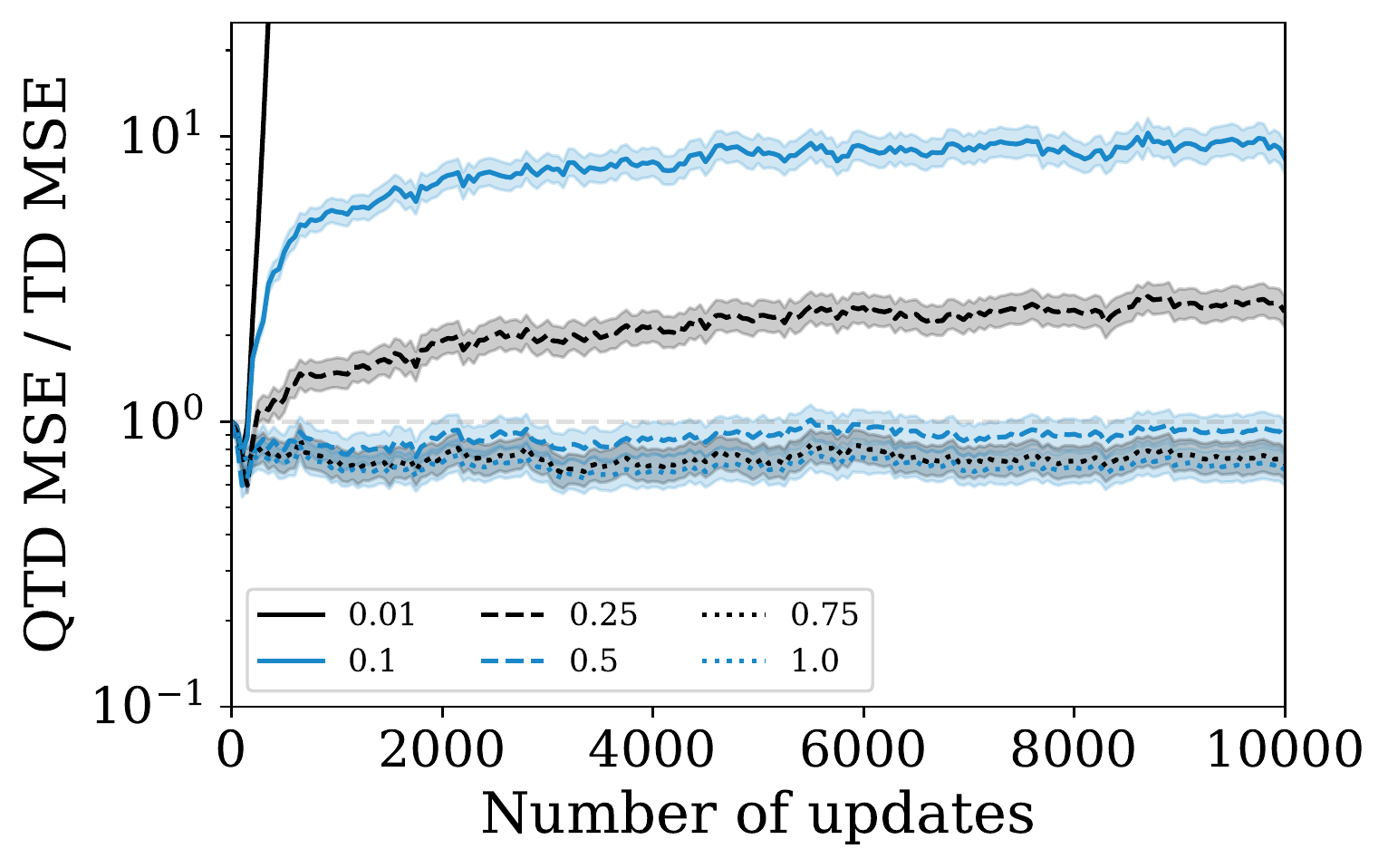}
    \caption{Performance of QTD(128) relative to TD, with optimal learning rates, on the environment with deterministic cycle transition structure, and Gaussian rewards of varying standard deviation, indicated in the legend.}
    \label{fig:reward-scale-sweep}
\end{figure}

\subsection{An ablation: Pseudo-quantile temporal-difference learning}\label{sec:new-algorithms}

We have motivated QTD theoretically as an effective algorithm for tabular policy evaluation, and we have also seen this borne out empirically. We have described its contrasting performance profile to TD, and noted its properties of (i) bounded-magnitude updates, and (ii) controllable fixed-point error. Taking a step back, a natural question to ask is: are there further nuances to the particular form of the QTD update which make it an effective algorithm? To investigate this question further, in this section we study a new algorithm for tabular policy evaluation, \emph{pseudo-quantile temporal-difference learning} (PQTD), which uses the same form of quantile updates as QTD, though does \emph{not} aim to learn quantiles of the return distribution. Our goal is to understand the role played by these two components of QTD in forming an effective policy evaluation algorithm.

In particular, motivated by \citeauthor{achab2020ranking}'s (\citeyear{achab2020ranking}) study of the one-step random return $R + \gamma V^\pi(X')$ (see also \citet{achab2021robustness},  \citet{achab2022distributional}, and \citet{achab2023one}), PQTD aims to learn the quantiles of the distribution of these random variables, rather than those of the usual return distribution. The approach is presented in Algorithm~\ref{alg:pqtd}. The distinction from QTD is that the targets in the quantile regression update are constructed from the mean-return estimate at the next state, rather than from the quantile estimates themselves; the learnt quantile estimates therefore reflect only the randomness resulting from a single step of environment interaction. This is also motivated by the approach of two-hot encoded categorical value learning in recent deep RL applications \citep{schrittwieser2020mastering,hessel2021muesli,hafner2023mastering}, which can be interpreted as a one-step version of categorical distributional RL \citep{bellemare2017distributional}.

\begin{algorithm}
    \begin{algorithmic}[1]
        \REQUIRE Quantile estimates $((\theta(x, i))_{i=1}^m : x \in \mathcal{X})$, observed transition $(x, r, x')$, learning rate $\alpha$.
        \FOR{$i=1,\ldots,m$}
            \STATE $ \theta'(x, i) \leftarrow  \theta(x,i) + $ \\
                $\ \alpha \Big( \tau_i - \mathbbm{1}[\theta(x, i) - r - \gamma  \frac{1}{m} \sum_{j=1}^m \theta(x', j) < 0 ] \Big)$
        \ENDFOR
        \STATE Set $\theta(x, i) \leftarrow \theta'(x, i)$ for $i=1,\ldots,m$
        \STATE \textbf{return} $((\theta(y, i))_{i=1}^m : y \in \mathcal{X})$
    \end{algorithmic}
    \caption{PQTD update.}
    \label{alg:pqtd}
\end{algorithm}

The results of running PQTD on the same suite of environments as reported in Figure~\ref{fig:main-sweep} are given in Figure~\ref{fig:main-sweep-pqtd}, with improvements at optimised learning rates across a range of number of updates displayed in Figure~\ref{fig:main-sweep-pqtd-over-updates}. Overall, similar behaviour is observed with PQTD as with QTD: larger learning rates to TD are preferred, and the approach tends to work best in the presence of high environment stochasticity. However, the level of performance obtained is generally somewhat worse than QTD, and worse than TD in several stochastic environments too; this discrepancy provides a useful opportunity to understand the success of QTD better.

\begin{figure}[t]
    \centering
    
    \includegraphics[width=.48\textwidth]{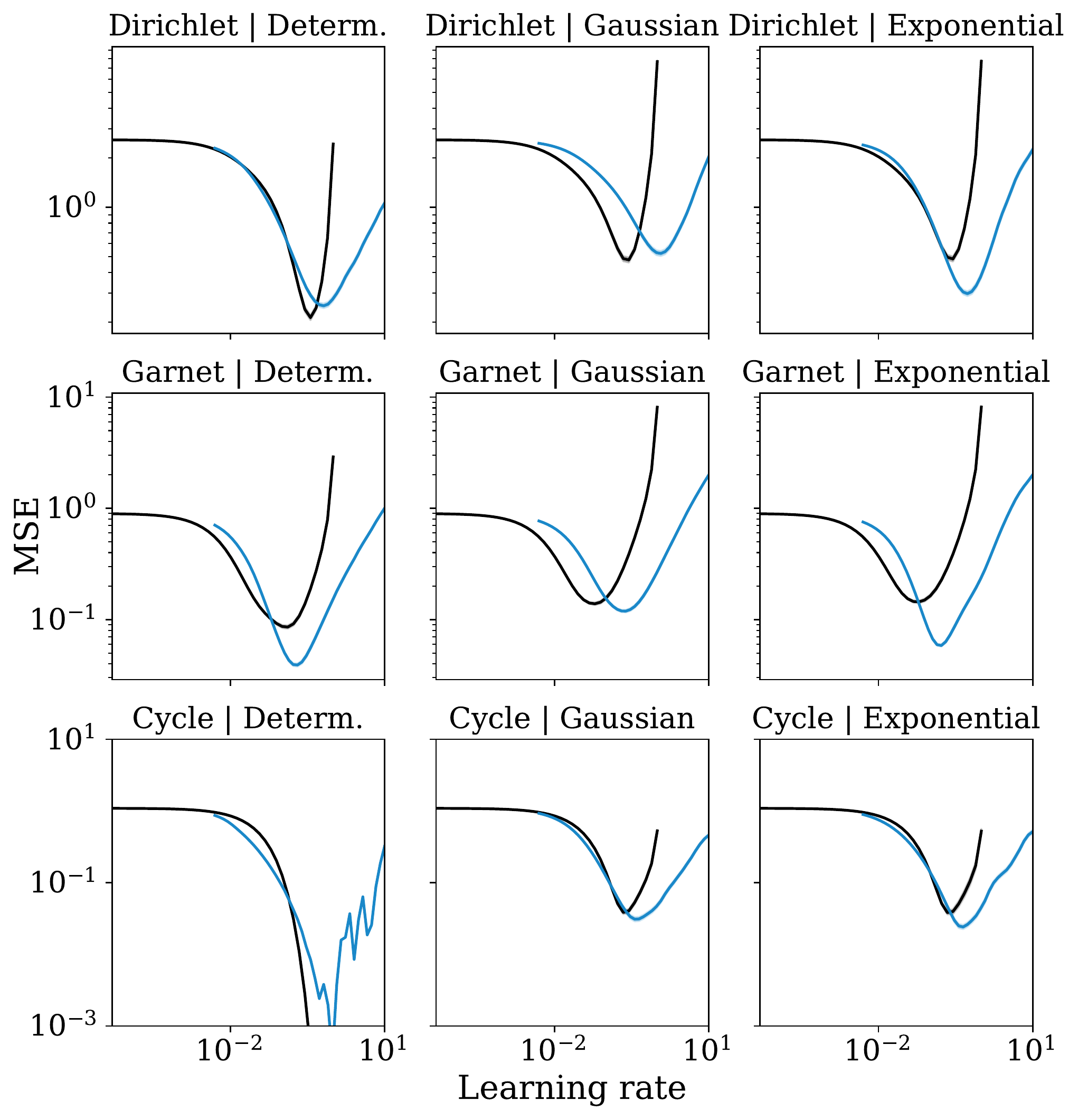}
    
    \caption{Mean-squared error against learning rate for TD (black) and PQTD(128) (blue).}
    \label{fig:main-sweep-pqtd}
\end{figure}

\begin{figure}[t]
    \centering
    
    \includegraphics[width=.48\textwidth]{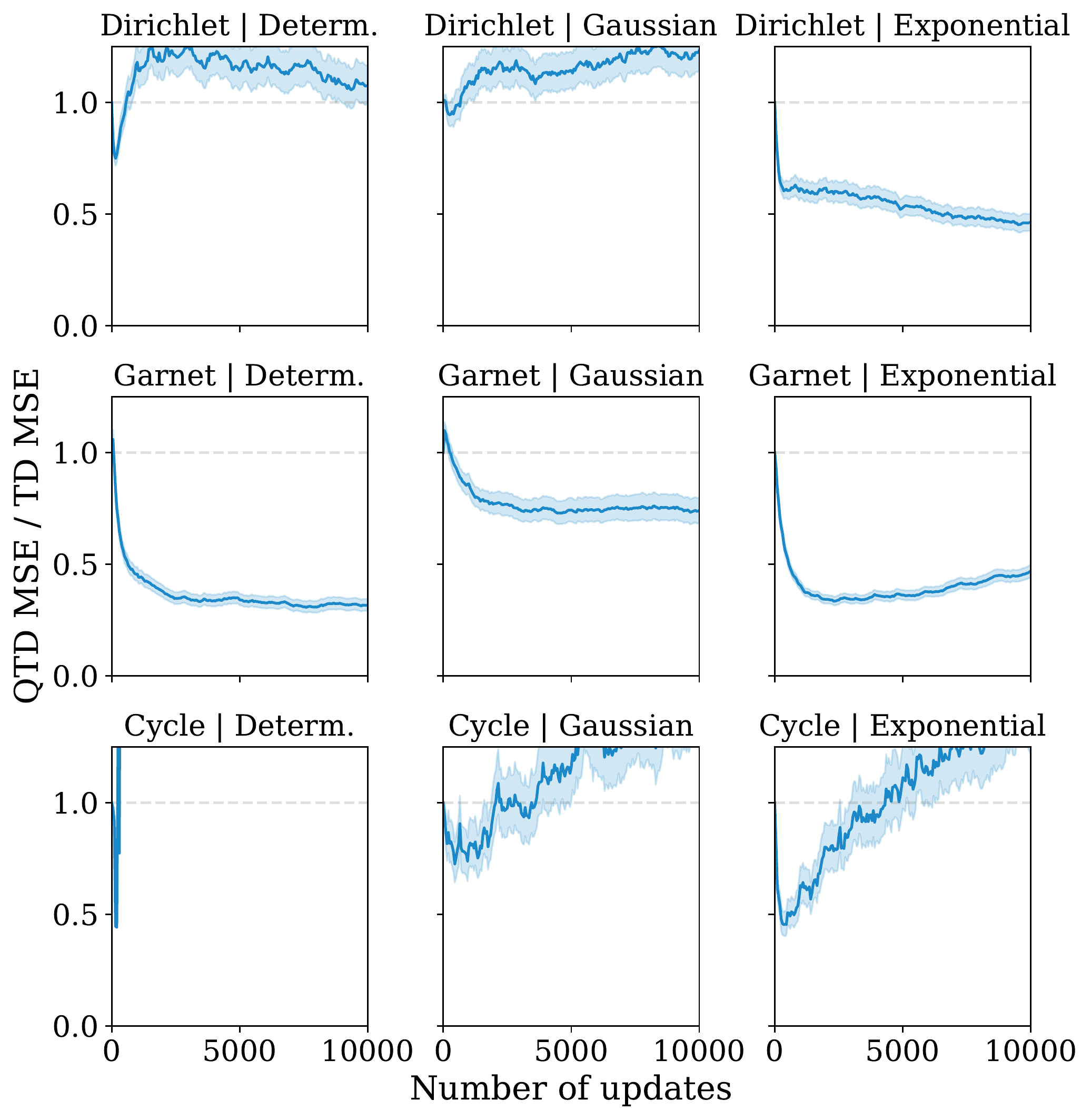}
    
    \caption{Relative improvement in MSE of PQTD over TD, with numbers of updates ranging from 0 to 10,000.}
    \label{fig:main-sweep-pqtd-over-updates}
\end{figure}

In particular, considering the cycle environment with Gaussian reward noise, the fixed-point bias for both QTD and PQTD is in fact zero in this case; for readers familiar with distributional dynamic programming \citep{bdr2022}, intuitively this follows from symmetry of the one-step target distributions, meaning that the average of the learnt quantiles is equal to the mean. Our earlier decomposition of the error therefore suggests that the discrepancy between QTD and PQTD must arise from finite-sample error, and points to differences in the update variance between the algorithms. Our empirical observations concur with this conjecture, with PQTD updates often having variance several times greater than those of QTD around the points of convergence. The form of the update for both QTD and PQTD is also informative; the term multiplying the learning rate in Algorithm~\ref{alg:pqtd} can take on only the values $\tau_i$ or $1-\tau_i$, whereas the averaging that occurs in the corresponding QTD update allows for significantly lower-magnitude updates, and hence potentially lower variance.

This finding suggests that a strength of QTD for value estimation is not only its bounded-magnitude updates, but the fact that variance of these updates is often significantly better than the bounds alone suggest, and that specifically learning the distribution of the full return can have beneficial variance-reduction properties in temporal-difference learning.

\section{Related work}\label{sec:related-work}

\textbf{Mean estimation with quantiles.} The approach of estimating a location parameter by averaging quantiles dates back at least to \citet{daniell1920observations}, who investigated the non-uniform averaging of order statistics to estimate a one-dimensional location parameter. \citet{mosteller1946some} developed this line of work further, investigating the statistical properties of averages of quantile estimates in greater detail. Interestingly, several proposals for which quantile levels should be averaged were made in this work, including the levels $\tau_i = \tfrac{2i-1}{2m}$ used by QTD, though without theoretical justification. \citet{gastwirth1966robust} also studied the efficiency of a mean estimator based on averaging of three specific quantiles for symmetric distributions with varying levels of heavy-tailedness. \citet{huber1964robust} proposed using smoothed versions of quantile losses for location estimation. See also \citet{andrews1972robustness} for a broader review of robust approaches to location estimation.
Online estimation of quantiles via incremental algorithms also has a long history; quantile estimation (in the supervised learning setting) is one of the examples provided by \citet{robbins1951stochastic} in their work introducing the field of stochsatic approximation.

\textbf{Deep quantile temporal-difference learning.} In addition to the original QTD algorithm \citep{dabney2018distributional}, recent theoretical developments \cite{lheritier2022cramer,rowland2023analysis}, and extensions in the context of deep reinforcement learning \citep{dabney2018implicit,yang2019fully}, several architectural innovations specifically exploiting neural network function approximation have been proposed to avoid the quantile-crossing problem when combining QTD with neural function approximation \citep{zhou2020non,luo2021distributional,theate2021distributional}.

\textbf{Distributional reinforcement learning algorithms.} In this paper, we have focused on quantile temporal-difference learning, a particular instance of a distributional reinforcement learning algorithm. Other distributional reinforcement learning algorithms include categorical temporal-difference learning \citep[CTD; ][]{bellemare2017distributional,rowland2018analysis}, maximum-mean discrepancy-based methods \citep{nguyen2021distributional}, methods using distributional representations based on mixtures of Gaussians \citep{barth2018distributed}, and methods using Sinkhorn divergences \citep{sun2022distributional}. It is interesting to contrast the finding that QTD is a strong algorithm for tabular policy evaluation, with properties that \emph{complement} those of TD, with prior findings relating to CTD \citep{rowland2018analysis,lyle2019comparative,bdr2022}. In contrast to QTD, this prior work showed that in many circumstances, CTD behaves \emph{identically} to TD for mean estimation, and so offers no additional benefit, or complementary profile of performance.

\textbf{Robust approaches to TD learning and optimisation.} 
A variety of approaches to robust and regularised variants of TD learning have been considered previously \citep{bossaerts2020exploiting,lu2021robust,meyer2021accelerated,klima2019robust, ghiassian2020gradient,liu2012regularized,manek2022pitfalls}. Bounded updates naturally arise from the QTD learning algorithm; bounded gradients are also commonly encountered in deep learning as a heuristic approach to stabilising optimisation through clipping \citep{mikolov2012statistical,pascanu2013difficulty}, as well as in fundamental optimisation algorithms \citep{riedmiller1993direct}. 

\section{Conclusion}

In this paper, we have shown that QTD can be viewed as a fundamental algorithm for policy evaluation, with complementary properties to the classical approach to temporal-difference learning. The theoretical and empirical analysis, as well as the introduction and study of the related algorithm PQTD, has given indications as to which kinds of environments we might expect one approach to improve over the other. We emphasise that these findings are of course not exhaustive, and we expect there to be significant value in further empirical investigation of QTD as a tabular policy evaluation algorithm, as well as analysis of variants incorporating aspects such as multi-step returns, off-policy corrections, and function approximation, all of which interact in various ways with the complementary trade-offs made by TD and QTD between fixed-point error, variance and expected update magnitude \citep{white2016greedy,mahmood2017multi,rowland2020adaptive}. Precise finite-sample bounds on performance are also a natural direction for future work.

These findings are also pertinent to the overarching questions as to where exactly the benefits of distributional RL stem from. Common hypotheses have often focused on the interaction between distributional predictions and non-linear function approximation, with mechanisms such as improved representation learning, prevention of rank collapse, and improved loss landscapes being proposed. This work highlights that even in risk-neutral tabular settings, there are benefits to taking a distributional approach to reinforcement learning, and opens up several directions of research to understand the role of distributional RL as a core technique in reinforcement learning. Historically, distributional RL algorithms have often been evaluated in (near-)deterministic environments; this paper also supports the idea that by evaluating algorithms on a wider range of environments, we may obtain a more nuanced view of the strengths and weaknesses of the algorithms at play.
Above all, this paper aims to show that distributional reinforcement learning has a fundamental role in developing algorithms that complement our existing approaches to core tasks such as policy evaluation; 
to estimate the mean, it can pay to estimate the full distribution.

\section*{Acknowledgements}

We thank David Abel for detailed comments on an earlier draft, and the reviewers \& area chair for their helpful comments on the paper.
The experiments in this paper were undertaken using
the Python 3 language, and made use of the NumPy \citep{harris2020array}, SciPy \citep{2020SciPy-NMeth}, and Matplotlib \citep{hunter2007matplotlib} libraries.

\clearpage

\bibliography{main}
\bibliographystyle{plainnat}

\clearpage

\onecolumn

\begin{appendix}

\section*{\centering APPENDICES}

We briefly summarise the contents of the appendices for convenience:
\begin{itemize}
    \item In Section~\ref{appendix:comp}, we provide further background on QTD, in particular describing the motivation for the form of the algorithm, key theoretical convergence results, and computational considerations.
    \item In Section~\ref{sec:proofs}, we provide proofs of all results stated in the main paper.
    \item In Section~\ref{sec:further-experiment-details}, we provide further details and context on the experimental results reported in the main paper.
    \item In Section~\ref{sec:app:further-experiments}, we provide further experimental results to complement those presented in the main paper.
\end{itemize}

\section{Further background on quantile temporal-difference learning}\label{appendix:comp}

As mentioned in the main paper, we encourage readers to consult \citet{rowland2023analysis} and \citet{bdr2022} for detailed background on the QTD algorithm. Here, we present a brief overview of key motivation, intuition, and theoretical results for the algorithm, following the general discussion in \citet{dabney2018distributional} and \citet{rowland2023analysis}.

\subsection{Motivation}

Quantile temporal-difference learning is motivated by the aim of learning certain quantiles of the return distribution at each state of the MDP.

\textbf{Quantiles.} For a probability distribution $\nu$ over the real numbers, with corresponding cumulative distribution function (CDF) $F$ which is continuous and strictly increasing, the $\tau$-quantile of $\nu$ (for $\tau \in (0,1)$) is the unique value $z$ such that $F(z) = \tau$. In other words, the $\tau$-quantile of this distribution is the value which a random sample from the distribution has probability exactly $\tau$ of being less than. In this way, quantiles can be thought of as an inverse to the CDF, and an alternative description of the probability distribution itself.

Our description above made the assumption that $F$ is continuous and strictly increasing. While this is true for many distributions of interest, such as Gaussians, there are also many distributions for which is is not true, such as distributions of random variables with only finitely many outcomes. The definition of $\tau$-quantile given above cannot apply directly to such CDFs, but the following generalisation applies to all distributions: For any distribution $\nu$ over the real numbers with corresponding cumulative distribution function (CDF) $F$, the \emph{set} of $\tau$-quantiles is defined to be the interval
\begin{align*}
    [ F^{-1}(\tau), \bar{F}^{-1}(\tau) ] \, ,
\end{align*}
where $F^{-1}(\tau) = \inf \{ z \in \mathbb{R} : F(z) \geq \tau \}$, and $\bar{F}^{-1}(\tau) = \inf \{ z \in \mathbb{R} : F(z) > \tau \}$.

\renewcommand*{\thefootnote}{\fnsymbol{footnote}}

\textbf{Quantile regression.} Quantile regression provides a means of approximately computing the quantiles of a distribution of interest, using samples from this distribution and running stochastic gradient descent on a particular loss function. Just as the mean of a probability distribution $\nu$ over the real numbers minimises the squared loss function $\theta \mapsto \mathbb{E}_{Z \sim \nu}[(Z - \theta)^2]$.\footnote{To be precise, this characterisation of the mean requires that $\nu$ have finite variance, so that the expectation appearing in the definition of the function is finite.}, the $\tau$-quantiles of a distribution $\nu$ are precisely the minimisers of the \emph{quantile regression loss}:
\begin{align}\label{eq:qr-loss}
    \theta \mapsto \mathbb{E}_{Z \sim \nu}[|Z - \theta| (\tau \mathbbm{1}[Z > \theta] + (1-\tau) \mathbbm{1}[Z < \theta])] \, ;
\end{align}
see e.g. \citet{koenker2005quantile} or \citet{koenker1978regression} for further background. Given an estimate $\theta$ for the $\tau$-quantile of $\nu$, and a sample $Z$ from $\theta$, a straightforward calculation shows that stochastic gradient descent (SGD) on this loss corresponds to the update
\begin{align}\label{eq:qr-sgd-update}
    \theta \leftarrow \theta + \alpha( \tau - \mathbbm{1}[Z < \theta]) \, ,
\end{align}
where $\alpha$ is the step size used in the SGD update. Since the quantile regression loss in  Equation~\eqref{eq:qr-loss} is convex in $\theta$, under mild conditions repeated application of the SGD update with a sequence of samples $(Z_i)_{i=1}^\infty$ drawn i.i.d.\ from $\nu$ results in the estimate $\theta$ converging to the set of $\tau$-quantiles with probability 1 (see e.g.\ \citeauthor{kushner2003stochastic}, \citeyear{kushner2003stochastic}), and this is therefore a sensible algorithm for learning quantiles from a streaming source of samples.

\textbf{QTD as a combination of quantile regression and bootstrapping. } The motivation behind QTD, as proposed by \citet{dabney2018distributional}, is to use the quantile regression update rule in Equation~\eqref{eq:qr-sgd-update} to learn estimates $(\theta(x, i))_{i=1}^m$ of quantiles $\tau_i = \tfrac{2i-1}{2m}$, $i=1,\ldots,m$ of the return distribution at each state $x \in \mathcal{X}$. If we had access to samples $\sum_{t\geq0} \gamma^t R_t$ drawn from the true distribution of returns at a state $x$, e.g. from full trajectories of interaction with the environment, then updates of the form given in Equation~\eqref{eq:qr-sgd-update} can be directly applied to each estimate $\theta(x, i)$ separately, leading to an update of the form
\begin{align*}
    \textstyle
    \theta(x, i) \leftarrow \theta(x, i) + \alpha( \tau_i - \mathbbm{1}[\sum_{t\geq0} \gamma^t R_t < \theta(x, i)]) \, .
\end{align*}
The idea behind QTD is to replace the random return $\sum_{t\geq0} \gamma^t R_t$ at the state $x$ with an alternative random variable that involves \emph{bootstrapping} (in the sense that the term is used in reinforcement learning, not statistics), in a similar way that classical temporal-difference learning modifies Monte Carlo learning \citep{sutton2018reinforcement}. Specifically, if the first transition in the trajectory beginning at state $x$ takes us to a new state $X'$, we can replace the portion of the discounted return from this state onwards by instead sampling one of our estimated quantiles of the return distribution at $X'$. Letting $J$ be a uniformly randomly chosen index in $\{1,\ldots,m\}$, this results in an update of the form
\begin{align*}
    \textstyle
    \theta(x, i) \leftarrow \theta(x, i) + \alpha( \tau_i - \mathbbm{1}[ R + \gamma \theta(X', J) < \theta(x, i)]) \, .
\end{align*}
Finally, averaging over the different possible choices for $J$ yields the update 
\begin{align*}
    \theta(x, i) \leftarrow \theta(x, i) + \alpha \Big( \tau_i -\frac{1}{m} \sum_{j=1}^m \mathbbm{1}[ R + \gamma \theta(X', j) < \theta(x, i)] \Big) \, ,
\end{align*}
which is precisely the QTD update presented in Algorithm~\ref{alg:qtd-mean}. Note that the new return estimator is generally not distributed according to the exact return, and so the QTD update does not directly inherit convergence guarantees from the stochastic gradient descent case described above. Issues of convergence are briefly covered in the section below.

In summary, we have motivated QTD as an algorithm that performs updates similar to those of stochastic gradient descent on the quantile regression, but which uses bootstrapped samples of the random return. This allows for, amongst other things, potential statistical benefits over Monte Carlo learning, as described in the recent work of \citet{cheikhi2023statistical}, as well as the computation of updates to quantile estimates after each transition in the environment, rather than needing to wait for full trajectories to become available in episodic environments.

\subsection{Convergence guarantees}\label{sec:app:convergence}

As the theoretical discussion in the main paper centres around convergence points of the QTD update, we review relevant convergence results here at the level of detail required for the paper; for full details, see \citet{rowland2023analysis}.

\citet{rowland2023analysis} show that under mild conditions, the estimates $(\theta_k(x, i))_{i=1}^m$ produced via $k$ QTD updates from some initial parameters $\theta_0$ are guaranteed to convergence with probability 1 as $k \rightarrow \infty$; note that these conditions include the requirement of decaying step sizes, as is often the case with stochastic approximation theory. Unlike classical temporal-difference learning, the estimates of QTD may converge to a \emph{set} of convergence points, rather than a single point; this phenomenon is related to the possibility of non-unique quantiles described above.

\citet{rowland2023analysis} show that each element of the set of convergence points is the fixed point of a projected distributional Bellman operator $\Pi^{\lambda} \mathcal{T}^\pi : \mathbb{R}^{\mathcal{X} \times [m]} \rightarrow \mathbb{R}^{\mathcal{X} \times [m]}$ that act on the quantile estimates. The precise details of these operators are used in the proof of Proposition~\ref{prop:abstract-fixed-point-quality} below, so we briefly recount their definitions and key properties. The distributional Bellman operator $\mathcal{T}^\pi$ is typically defined on the space of return-distribution functions $\mathscr{P}(\mathbb{R})^\mathcal{X}$ (that is, collections of distributions indexed by state), and outputs the corresponding distributional Bellman targets. Given $\eta \in \mathscr{P}(\mathbb{R})^\mathcal{X}$, we have
\begin{align*}
    (\mathcal{T}^\pi \eta)(x) = \mathcal{D}_\pi(R + \gamma G(X')) \, ,
\end{align*}
where $(G(y) : y \in \mathcal{X})$ are a collection of random variables with $G(y) \sim \eta(y)$ for all $y \in \mathcal{X}$, $(x, R, X')$ is a random transition beginning at $x$, and $\mathcal{D}_\pi$ extracts the distribution of the input random variable when the random transition is generated according to $\pi$. We can also overload this notation so that $\mathcal{T}^\pi$ takes quantile estimates $((\theta(x, i))_{i=1}^m : x \in \mathcal{X})$ as inputs, first mapping the quantile estimates to the return-distribution function $\eta$ given by
\begin{align*}
    \eta(x) = \sum_{i=1}^m \frac{1}{m}  \delta_{\theta(x, i)} \, ,
\end{align*}
and then applying the standard definition of the distributional Bellman operator above. With this convention, we have
\begin{align*}
    (\mathcal{T}^\pi \theta)(x) = \mathcal{D}_\pi(R + \gamma \theta(X', J)) \, ,
\end{align*}
where $J$ is a uniformly random index in $\{1,\ldots,m\}$. The projection operator $\Pi^\lambda$, parametrised by $\lambda \in [0,1]^{\mathcal{X} \times [m]}$, then extracts quantiles from the distributions defined by $\mathcal{T}^\pi \theta$, at the levels $\tau_i=\tfrac{2i-1}{2m}$, for $i=1,\ldots,m$. Mathematically, we have
\begin{align*}
    (\Pi^\lambda \mathcal{T}^\pi\theta )(x, i) = (1 - \lambda(x, i)) F^{-1}_{(\mathcal{T}^\pi \theta)(x)}(\tau_i) + \lambda(x, i) \bar{F}^{-1}_{(\mathcal{T}^\pi \theta)(x)}(\tau_i) \, ,
\end{align*}
where we write $F_\nu$ for the CDF associated with distribution $\nu$.
\citet{rowland2023analysis} show that $\Pi^\lambda \mathcal{T}^\pi$ is contractive under a certain Wasserstein metric, and by appealing to the Banach fixed point theorem in an appropriate space of return-distribution functions, show that it has a fixed point. The possibility of multiple convergence points as $\lambda$ varies is a distinctive property of QTD, and not present in non-distributional algorithms such as classical TD learning. The analysis in this paper relating to these convergence points, such as analysis of fixed-point error, applies to each such point individually, and therefore the question of uniqueness of convergence points is secondary.

\subsection{Computational complexity}

To complement the description of QTD given above, this section provides further discussion on the computational properties of the algorithm, drawing comparisons with TD. 

\textbf{Space complexity.} In addition to the linear scaling with the size of the state space that is common to tabular algorithms, QTD additionally scales linearly with the number of quantiles, $m$, to be predicted at each state.

\textbf{Time complexity.} An implementation of Algorithm~\ref{alg:qtd-mean} implementing the sum appearing in Line~\ref{algline:update} as a for-loop has a time complexity of $O(m^2 T)$. One factor of $m$ arises from the fact that there are $m$ quantiles to be updated at each state $x_t$, and the second factor of $m$ arises from the fact that there are $m$ quantiles at the target state $x'_t$ for which to compute the temporal-difference error.

However, the update for index $i$ in Line~\ref{algline:update} of Algorithm~\ref{alg:qtd-mean} depends only on whether $\gamma \theta(x'_t, j)$ is greater than $\theta(x_t, i) - r_t$ for each $j$. This can be exploited to produce an alternative implementation requiring only $O(T m \log m)$ time. This implementation first sorts the list $(\gamma \theta(x'_t, j))_{j=1}^m$ in time $O(m \log m)$. Then, for each index $i$, we can temporarily insert $\theta(x_t, i) - r_t$ into this sorted list at a cost of $O(\log m)$ (via e.g. binary search), and the index of the inserted term reveals exactly how many of the bootstrap terms $(\gamma \theta(x'_t, j))_{j=1}^m$ are less than $\theta(x_t, i) - r_t$, from which the update in Line~\ref{algline:update} of Algorithm~\ref{alg:qtd-mean} can be immediately computed. This leads to a per-update time complexity of $O(m \log m)$, improving on the $O(m^2)$ time complexity of the first implementation described above. In contrast to QTD, PQTD requires only linear time to compute each update.

\section{Proofs}\label{sec:proofs}

\subsection{Proof of Proposition~\ref{prop:qtd-fixed-point-bound}}

\propQTDFixedPoint*

\begin{proof}
    By Proposition~6.1 of \citet{rowland2023analysis}, we have that for any fixed point $\theta^{\text{QTD}}_m \in \mathbb{R}^{\mathcal{X} \times m}$ of QTD($m$), the corresponding return distribution estimate
    \begin{align*}
        \eta(x) = \frac{1}{m} \sum_{i=1}^m \delta_{\theta^{\text{QTD}}_m(x, i)}
    \end{align*}
    satisfies
    \begin{align*}
        w_1(\eta(x), \eta^\pi(x)) \leq \frac{R_\text{max} - R_\text{min}}{2m (1-\gamma)^2} \, ,
    \end{align*}
    for all $x \in \mathcal{X}$, where $\eta^\pi$ is the true return-distribution function, and $w_1$ is the Wasserstein-1 distance. Since the Wasserstein-1 distance between two distributions bounds the distance between their means \citep[see e.g. ][]{villani2009optimal}, we therefore have
    \begin{align*}
        \Big | \frac{1}{m} \sum_{i=1}^m \theta^{\text{QTD}}_m(x,i) - V^\pi(x)  \Big | \leq \frac{R_\text{max} - R_\text{min}}{2m (1-\gamma)^2} \, ,
    \end{align*}
    for all $x \in \mathcal{X}$, as required.
\end{proof}

\subsection{Proof of Proposition~\ref{prop:qtd-fixed-subgaussian}}

\propQTDSubGaussian*

We will establish the proof of this result by first deriving the more general Proposition~\ref{prop:abstract-fixed-point-quality}. The proof of this result modifies the approach taken in the proof of Proposition~6.1 from \citet{rowland2023analysis} in the case of bounded rewards.

\begin{proposition}\label{prop:abstract-fixed-point-quality}
    Consider an MDP with all reward distributions having $\tfrac{1}{2m}$-quantiles at least $\ubar{q}$ and $(1-\tfrac{1}{2m})$-quantiles at most $\bar{q}$.
    Then, for any convergence point $\theta^{\text{QTD}}_m$ of QTD($m$) with corresponding return-distribution function $\eta^{\text{QTD}}_m \in \mathscr{P}(\mathbb{R})^{\mathcal{X}}$, so that
    \begin{align*}
        \eta^{\text{QTD}}_m(x) = \sum_{i=1}^m \frac{1}{m} \delta_{\theta^{\text{QTD}}_m(x, i)} \, ,
    \end{align*}
    we have
    \begin{align*}
        w_1(\eta^{\text{QTD}}_m(x), \eta^\pi(x)) \leq \frac{1}{1-\gamma}\left( \frac{\bar{q} - \ubar{q}}{2m(1-\gamma)} + \mathbb{E}^\pi_x[ (R - \bar{q}) \mathbbm{1}[ R > \bar{q}]] -   \mathbb{E}^\pi_x[ (R - \ubar{q}) \mathbbm{1}[ R < \ubar{q} ] ] \right)
    \end{align*}
    for each $x \in \mathcal{X}$.
\end{proposition}
\begin{proof}
    As described in Section~\ref{sec:app:convergence} (see also \citeauthor{rowland2023analysis}, \citeyear{rowland2023analysis}), let $\Pi^\lambda \mathcal{T}^\pi$ be a projected distributional Bellman operator with $\theta^{\text{QTD}}_m$ as a fixed point, so that $\Pi^\lambda \mathcal{T}^\pi \eta^{\text{QTD}}_m = \eta^{\text{QTD}}_m$. Defining the sup-Wasserstein-1 distance on $\mathscr{P}(\mathbb{R})^{\mathcal{X}}$ (see e.g. Chapter~4, \citeauthor{bdr2022}, \citeyear{bdr2022}) by
    \begin{align*}
        \overline{w}_1(\eta, \eta) = \max_{x \in \mathcal{X}} w_1(\eta(x), \eta'(x)) \, ,
    \end{align*}
    we have
    \begin{align*}
        \overline{w}_1(\eta^{\text{QTD}}_m, \eta^\pi) \leq & \overline{w}_1(\eta^{\text{QTD}}_m, \mathcal{T}^\pi \eta^{\text{QTD}}_m) + \overline{w}_1(\mathcal{T}^\pi \eta^{\text{QTD}}_m, \eta^\pi) \\
        = & \overline{w}_1(\Pi^\lambda \mathcal{T}^\pi \eta^{\text{QTD}}_m, \mathcal{T}^\pi \eta^{\text{QTD}}_m) + \overline{w}_1(\mathcal{T}^\pi \eta^{\text{QTD}}_m, \mathcal{T}^\pi \eta^\pi) \\
        = & \overline{w}_1(\Pi^\lambda \mathcal{T}^\pi \eta^{\text{QTD}}_m, \mathcal{T}^\pi \eta^{\text{QTD}}_m) + \gamma \overline{w}_1(\eta^{\text{QTD}}_m, \eta^\pi) \, ,
    \end{align*}
    with the first line following from the triangle inequality, 
    the first equality following since $\Pi^\lambda \mathcal{T}^\pi \eta^{\text{QTD}}_m = \eta^{\text{QTD}}_m$ and $\eta^\pi = \mathcal{T}^\pi \eta^\pi$, and the second line following from $\gamma$-contractivity of $\mathcal{T}^\pi$ in $\overline{w}_1$ \citep[see e.g.\ Proposition~4.15, ][]{bdr2022}.
    Rearranging, we obtain
    \begin{align*}
        \overline{w}_1(\eta^{\text{QTD}}_m, \eta^\pi) \leq \frac{1}{1 - \gamma}\overline{w}_1(\Pi^\lambda \mathcal{T}^\pi \eta^{\text{QTD}}_m, \mathcal{T}^\pi \eta^{\text{QTD}}_m) \, .
    \end{align*}
    To bound the right-hand term, we first reason about the support of $\eta^{\text{QTD}}_m$. Writing $\eta_0(x) = \delta_0$ for all $x \in \mathcal{X}$, and inductively defining $\eta_{k+1} = \Pi^\lambda \mathcal{T}^\pi \eta_k$ for $k \geq 0$, we have $\eta_{k+1} \rightarrow \eta^{\text{QTD}}_m$ in $\overline{w}_1$. We now prove by induction that $\eta_k(x)$ is supported on $[\ubar{q}\tfrac{1-\gamma^k}{1-\gamma}, \bar{q}\tfrac{1-\gamma^k}{1-\gamma}]$ for all $x \in \mathcal{X}$ and $k \geq 0$, from which it follows that $\eta^{\text{QTD}}_m(x)$ is supported on $[\ubar{q}\tfrac{1}{1-\gamma}, \bar{q}\tfrac{1}{1-\gamma}]$ for all $x \in \mathcal{X}$. The claim is straightforward to see for $k=0,1$. For the inductive step, we suppose $\eta_k(x)$ is supported on $[\ubar{q}\tfrac{1-\gamma^k}{1-\gamma}, \bar{q}\tfrac{1-\gamma^k}{1-\gamma}]$. Now, an instantiation of $(\mathcal{T}^\pi\eta_{k})(x)$ is given by $R + \gamma G$, where $(R, X') \sim P^\pi(\cdot|x)$, and $G | R, X' \sim \eta(X')$. Since $G \geq \ubar{q}\tfrac{1-\gamma^k}{1-\gamma}$ almost surely by hypothesis, we have $R + \gamma G \geq R + \gamma \ubar{q}\tfrac{1-\gamma^k}{1-\gamma}$ almost surely. By the definition of $\ubar{q}$, we have that the $1/(2m)$-quantile of this distribution is bounded below by $\ubar{q} + \gamma \ubar{q}\tfrac{1-\gamma^k}{1-\gamma} = \ubar{q}\tfrac{1-\gamma^{k+1}}{1-\gamma}$, and hence the support of $\Pi^\lambda \mathcal{T}^\pi \eta_k$ is bounded below by $\ubar{q}\tfrac{1-\gamma^{k+1}}{1-\gamma}$. The argument for the upper bound on the support is analogous, and the inductive claim is proven.
    
    Now, define $\Pi_B : \mathscr{P}(\mathbb{R}) \rightarrow \mathscr{P}(\mathbb{R})$ as the transformation of distributions that ``clips'' the support of a distribution to lie on the interval $[\ubar{q}\tfrac{1}{1-\gamma}, \bar{q}\tfrac{1}{1-\gamma}]$. Formally, if given a probability distribution $\nu \in \mathscr{P}(\mathbb{R})$ we write $\tilde{\nu}$ for its \emph{restriction} to the interval $(\ubar{q}\tfrac{1}{1-\gamma}, \bar{q}\tfrac{1}{1-\gamma})$, so that $\tilde{\nu}(A) = \nu(A \cap (\ubar{q}\tfrac{1}{1-\gamma}, \bar{q}\tfrac{1}{1-\gamma}))$, then $\Pi_B \nu$ is given by
    \begin{align*}
        \Pi_B \nu = \nu((-\infty, \ubar{q}\tfrac{1}{1-\gamma}]) \delta_{\ubar{q}\tfrac{1}{1-\gamma}} + \tilde{\nu} + \nu([ \bar{q}\tfrac{1}{1-\gamma}, \infty)) \delta_{\bar{q}\tfrac{1}{1-\gamma}} \, .
    \end{align*}
    Equivalently, $\Pi_B$ can be defined as the pushforward of $\nu$ through the function $f(z) = \max(\min(z, \bar{q}\tfrac{1}{1-\gamma}), \ubar{q}\tfrac{1}{1-\gamma})$, i.e. $\Pi_B \nu = f_\# \nu$. We then use the triangle inequality to write
    \begin{align*}
        \overline{w}_1(\Pi^\lambda \mathcal{T}^\pi \eta^{\text{QTD}}_m, \mathcal{T}^\pi \eta^{\text{QTD}}_m) \leq \overline{w}_1(\Pi^\lambda \mathcal{T}^\pi \eta^{\text{QTD}}_m, \Pi_{B} \mathcal{T}^\pi \eta^{\text{QTD}}_m) + \overline{w}_1(\Pi_{B} \mathcal{T}^\pi \eta^{\text{QTD}}_m, \mathcal{T}^\pi \eta^{\text{QTD}}_m) \, .
    \end{align*}
    Now, each distribution $(\mathcal{T}^\pi \eta^{\text{QTD}}_m)(x)$ has at most $1/(2m)$ mass in the region $(-\infty, \ubar{q}/(1-\gamma)]$, and similarly for the region $[\bar{q}/(1-\gamma), \infty)$, so $\Pi^\lambda \Pi_B \mathcal{T}^\pi \eta^{\text{QTD}}_m(x) = \Pi^\lambda \mathcal{T}^\pi \eta^{\text{QTD}}_m(x)$, and so we obtain the bound
    \begin{align*}
        \overline{w}_1(\Pi^\lambda \mathcal{T}^\pi \eta^{\text{QTD}}_m, \Pi_{B} \mathcal{T}^\pi \eta^{\text{QTD}}_m) \leq \frac{\bar{q} - \ubar{q}}{2m(1-\gamma)}
    \end{align*}
    from the argument given by \citet[Proposition~6.1, ][]{rowland2023analysis}.
    To quantify the contribution of the clipping to the second Wasserstein distance term, we first bound the transport cost for the upper clipping as
    \begin{align*}
        & \mathbb{E}[(R + \gamma \theta(X', J) - \bar{q}/(1-\gamma)) \mathbbm{1}[R + \gamma \theta(X', J) > \bar{q}/(1-\gamma)]] \\
        \leq & \mathbb{E}[(R + \gamma \bar{q}/(1-\gamma) - \bar{q}/(1-\gamma)) \mathbbm{1}[R + \gamma \theta(X', J) > \bar{q}/(1-\gamma)]] \\
        \leq & \mathbb{E}[(R + \gamma \bar{q}/(1-\gamma) - \bar{q}/(1-\gamma)) \mathbbm{1}[R + \gamma \bar{q}/(1-\gamma) > \bar{q}/(1-\gamma)]] \\
        = & \mathbb{E}[(R - \bar{q}) \mathbbm{1}[R > \bar{q}]] \, .
    \end{align*}
    The bound for the transport cost associated with the lower clipping is derived analogously, and the result follows.
\end{proof}

\begin{corollary}\label{corr:value-bound}
    Under the same conditions as Proposition~\ref{prop:abstract-fixed-point-quality}, we have that $V^{\text{QTD}}_m$, the value prediction obtained from the fixed point $\theta^{\text{QTD}}_m$, satisfies 
    \begin{align*}
       |V^{\text{QTD}}(x) - V^\pi(x)| \leq \frac{1}{1-\gamma}\left( \frac{\bar{q} - \ubar{q}}{2m(1-\gamma)} + \mathbb{E}^\pi_x[ (R - \bar{q}) \mathbbm{1}[ R > \bar{q}]] -   \mathbb{E}^\pi_x[ (R - \ubar{q}) \mathbbm{1}[ R < \ubar{q} ] ] \right)
    \end{align*}
    for each $x \in \mathcal{X}$.
\end{corollary}
\begin{proof}
    As in the proof of Proposition~\ref{prop:qtd-fixed-point-bound}, we use the fact that the Wasserstein-1 distance between two distributions bounds the difference in mean of the distributions, and the result now follows from Proposition~\ref{prop:abstract-fixed-point-quality}.
\end{proof}

Proposition~\ref{prop:abstract-fixed-point-quality} is stated abstractly, and to obtain concrete bounds for reward distributions of interest, it is necessary to obtain bounds on the quantiles and conditional expectations that feature in the result. We give a concrete result below for sub-Gaussian distributions, which proves Proposition~\ref{prop:qtd-fixed-subgaussian}.

\begin{proposition}\label{prop:general-subgaussian}
    Consider an MDP with all reward distributions having means in $[R_{\text{min}}, R_{\text{max}}]$, and all sub-Gaussian with parameter $\sigma^2$, so that $\mathbb{E}^\pi_x[\exp(\lambda (R-\mathbb{E}^\pi_x[R]))] \leq \exp(\lambda^2 \sigma^2/2)$, for all $x \in \mathcal{X}$. Then
    \begin{align*}
        \overline{w}_1(\eta^{\text{QTD}}_m, \eta^\pi)  \leq \frac{1}{(1-\gamma)m} \left( \frac{R_\text{max} - R_\text{min} + 2 \sigma \sqrt{2 \log (2m)}}{2(1-\gamma)} + \frac{\sigma}{\sqrt{2 \log(2m)}}  \right) \, ,
    \end{align*}
    and 
    \begin{align*}
        \| \hat{V} - V^\pi \|_\infty  \leq \frac{1}{(1-\gamma)m} \left( \frac{R_\text{max} - R_\text{min} + 2 \sigma \sqrt{2 \log (2m)}}{2(1-\gamma)} + \frac{\sigma}{\sqrt{2 \log(2m)}}  \right) \, .
    \end{align*}
\end{proposition}

\begin{proof}
    By a standard conversion of the sub-Gaussian moment-generating function condition into a concentration inequality, we obtain bounds for $\bar{q}$ and $\ubar{q}$ in Proposition~\ref{prop:abstract-fixed-point-quality} of the form
    \begin{align*}
        \bar{q} = R_{\text{max}} + \sigma \sqrt{2 \log (2m)} \, , \quad \ubar{q} = R_{\text{min}} - \sigma \sqrt{2 \log (2m)} \, .
    \end{align*}
    Next, we can also compute a bound on the expectations appearing in Proposition~\ref{prop:abstract-fixed-point-quality} as follows.
    \begin{align*}
        \mathbb{E}[(R - \bar{q}) \mathbbm{1}[R > \bar{q}]]
        \leq & \int_{\sigma \sqrt{2 \log (2m)}}^\infty \exp(-t^2/(2\sigma^2)) \mathrm{d}t \\
        \leq & \int_{\sigma \sqrt{2 \log (2m)}}^\infty \frac{t}{\sigma \sqrt{2 \log (2m)}} \exp(-t^2/(2\sigma^2)) \mathrm{d}t \\
        = & \frac{\sigma}{\sqrt{2 \log(2m)}} \exp( - (\sigma \sqrt{2 \log(2m)})^2 / (2\sigma^2) ) \\
        = & \frac{\sigma}{2m\sqrt{2 \log(2m)}} \, ,
    \end{align*}
    where the first inequality follows from the bounds on the quantiles and CDF tails for the sub-Gaussian reward distribution, and the second inequality is a standard trick for bounding tails of a Gaussian CDF. The same bound for the second expectation is derived analogously, and we obtain the required statement by substituting into the expression in Proposition~\ref{prop:abstract-fixed-point-quality}.
\end{proof}

\subsection{Proof of Proposition~\ref{prop:qtd-finite-mean}}

\propQTDFiniteMean*

To prove this result, we first require the following lemma, which controls how fast the tails of a distribution's CDF can decay if the distribution has finite mean.

\begin{lemma}\label{lem:cdf-bounds}
    If $F$ is the CDF of a probability distribution over the real numbers with finite mean, then we must have
    \begin{align*}
        F^{-1}(1 - \tfrac{1}{2m}) = o(m) \, , \quad F^{-1}(\tfrac{1}{2m}) = o(m) \, .
    \end{align*}
\end{lemma}
\begin{proof}
    We will prove the claim for $F^{-1}(1 - \tfrac{1}{2m})$; the proof for $F^{-1}(\tfrac{1}{2m})$ is analogous. First, note that for a random variable $Z$ with CDF $F$, we have
    \begin{align*}
        \mathbb{E}[Z \mathbbm{1}[Z > 0]] = \int_{F(0)}^1 F^{-1}(\tau) \; \mathrm{d}\tau \, .
    \end{align*}
    Suppose for a contradiction the growth bound on $F^{-1}(1 - \tfrac{1}{2m})$ does not hold. Then there exists $c>0$ such that for infinitely many $m$, $F^{-1}(1 - \tfrac{1}{2m}) > cm$. We will use this fact to lower-bound the value of the integral above. Note that for each such $m$, the rectangle $C_m = [1 - \tfrac{1}{2m}, 1] \times [0, F^{-1}(1-\tfrac{1}{2m})]$ lies between the curve to be integrated and the x-axis, and has area at least $c/2$. Now we pick a subsequence of these integers $m_1,m_2,\ldots$, with the property that the area of $C_{m_{k+1}} \cap C_{m_k}$ is less than $c/4$. Concretely, this can be achieved by taking $m_{k+1} > 2 c^{-1} F^{-1}(1 - \frac{1}{2m_k})$. We then obtain the bound
    \begin{align*}
        \int_{F(0)}^1 F^{-1}(\tau) \; \mathrm{d}\tau \geq \text{Area}(C_1) + \sum_{k=2}^\infty \text{Area}(C_{m_k} \setminus C_{m_{k-1}}) \geq c/2 + c/4 + c/4 + \cdots = \infty \, .
    \end{align*}
    However, this contradicts our initial assumption that the distribution has finite mean, and so we conclude that we must have $F^{-1}(1 - \tfrac{1}{2m}) = o(m)$, as required.
\end{proof}

\begin{proof}[Proof of Proposition~\ref{prop:qtd-finite-mean}]
    By Corollary~\ref{corr:value-bound} and Lemma~\ref{lem:cdf-bounds}, for $m$ sufficiently large we have
    \begin{align*}
        | V^{\text{QTD}}_m(x) - V^\pi(x) | \leq& \frac{1}{1-\gamma} \Big(\frac{o(m)}{2m(1-\gamma)} + \\
        & \ \mathbb{E}^\pi_x[ (R - F^{-1}(1 - \tfrac{1}{2m})) \mathbbm{1}[ R > F^{-1}(1 - \tfrac{1}{2m})]] -   \mathbb{E}^\pi_x[ (R - F^{-1}(\tfrac{1}{2m})) \mathbbm{1}[ R < F^{-1}(\tfrac{1}{2m}) ] ] \Big) \rightarrow 0 \, ,
    \end{align*}
    as required, where $F$ above is the CDF of the reward distribution at state $x$.
\end{proof}

\section{Further experimental details}\label{sec:further-experiment-details}

We provide full details for the experiments reported in the main paper.

\textbf{Environment details.}
\begin{itemize}
    \item Dense stochastic transition structure. We generate a transition matrix for a 20 state MDP by sampling each row of the transition matrix independently from a Dirichlet$(1,\ldots,1)$ distribution.
    \item Sparse transition structure. We generate a transition matrix for a 20 state MDP according to the Garnet protocol \citep{archibald1995generation}. Specifically, for each state, we independently sample 6 states (without replacement) to have non-zero transition probability, and allocate the transition probability uniformly across these states.
    \item Deterministic transition structure. We use a deterministic cycle transition structure over 10 states.
\end{itemize}

\textbf{Reward distributions.}
\begin{itemize}
    \item Deterministic distributions. At the time of generating the environment, we sample the reward at each state independently from a standard normal $N(0,1)$ distribution.
    \item Gaussian distributions. Mean rewards are sampled as above, and the distributions themselves are Gaussian with standard deviation 1.
    \item Exponential distributions. Mean rewards are sampled as above, and the distributions themselves are shifted Exponential(1) distributions with the specified means.
    \item $t$-distributions. Mean rewards are sampled as above, and the distributions themselves are shifted $t_2$-distributions with the specified means.
\end{itemize}

\textbf{Hyperparameters.} In all experiments, we use a default discount factor of $\gamma = 0.9$. For both TD and QTD methods, all predictions are initialised to 0.

\textbf{Mean-squared error measurement.} Each configuration was run 1,000 times, and the reported results are the mean-squared errors averaged over these 1,000 runs; error bars in the plots correspond to plus/minus two times the empirical standard error.

\textbf{Learning rates.} For TD, 40 learning rates are swept over the range $[5 \times 10^{-4}, 1]$, equally spaced in log-space. For QTD, 40 learning rates are swept over the range $[5 \times 10^{-3}, 10]$, equally spaced in log-space.

\section{Further experimental results}\label{sec:app:further-experiments}

Here, we report further comparisons to complement the results in the main paper.

\subsection{Optimal learning rate selection for Figure~\ref{fig:main-sweep-over-updates}}

To aid interpretation of Figure~\ref{fig:main-sweep-over-updates}, we display in Figure~\ref{fig:main-sweep-over-updates-lr} the optimal learning rates selected by TD and QTD for each number of updates. Here, the confidence bands indicate the range of learning rates for which the lower end of the MSE confidence region was smaller than the upper end of the confidence region for the actual chosen optimal learning rate.

\begin{figure}[h]
    \centering
    
    \includegraphics[width=.48\textwidth]{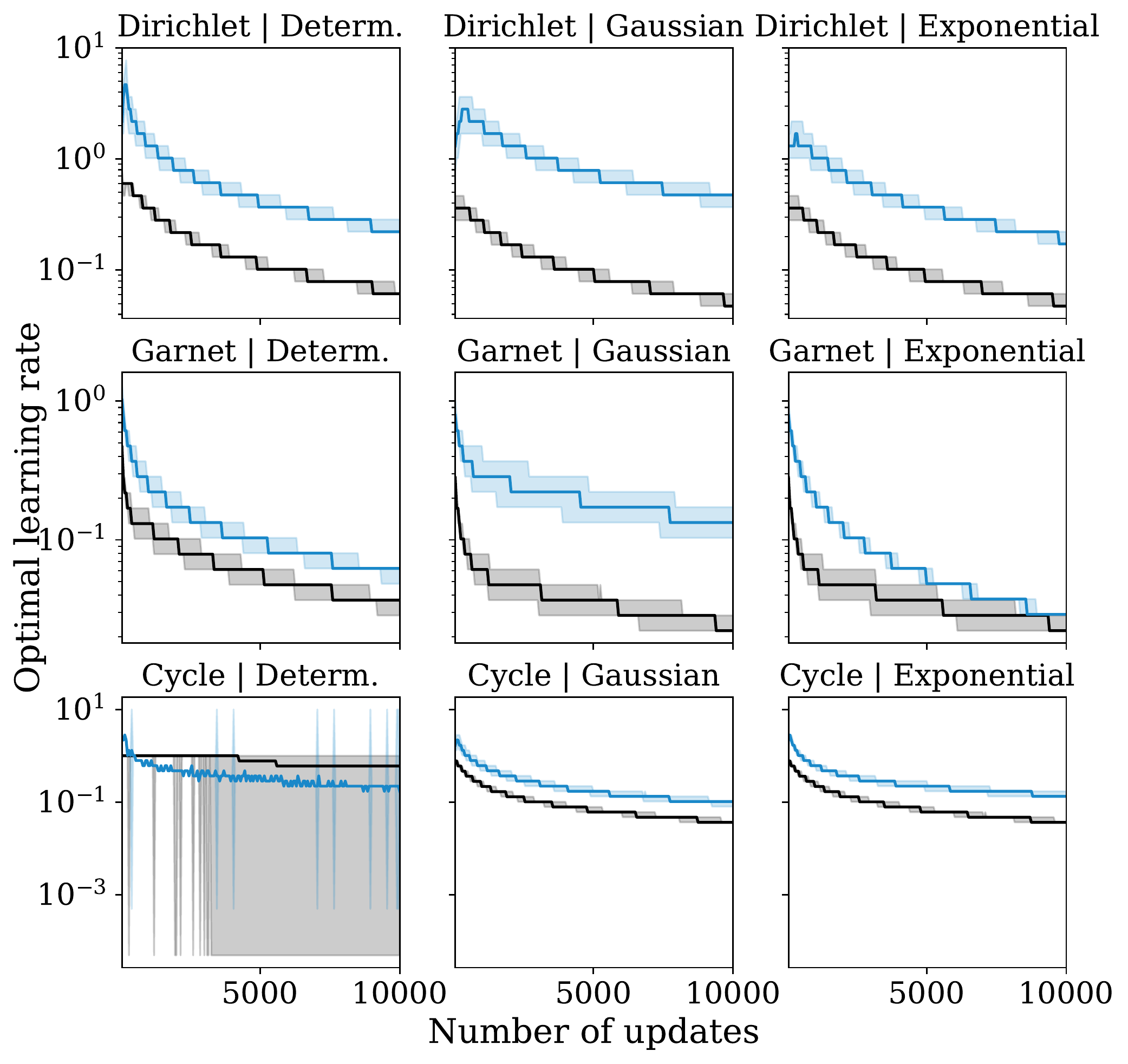}
    
    \caption{Optimal learning rate, as a function of number of updates, for both TD and QTD, for the results displayed in Figure~\ref{fig:main-sweep-over-updates}.}
    \label{fig:main-sweep-over-updates-lr}
\end{figure}

\subsection{Further results for $t_2$-distributed rewards}\label{sec:app:t}

To complement Figure~\ref{fig:t}, we plot mean-squared error against learning rates for both TD and QTD run with 1,000 updates in Figure~\ref{fig:t-lr}, in particular providing an indication of the typical magnitude of MSE that is attained by each method.

\begin{figure}[h]
    \centering
    
    \includegraphics[width=.6\textwidth]{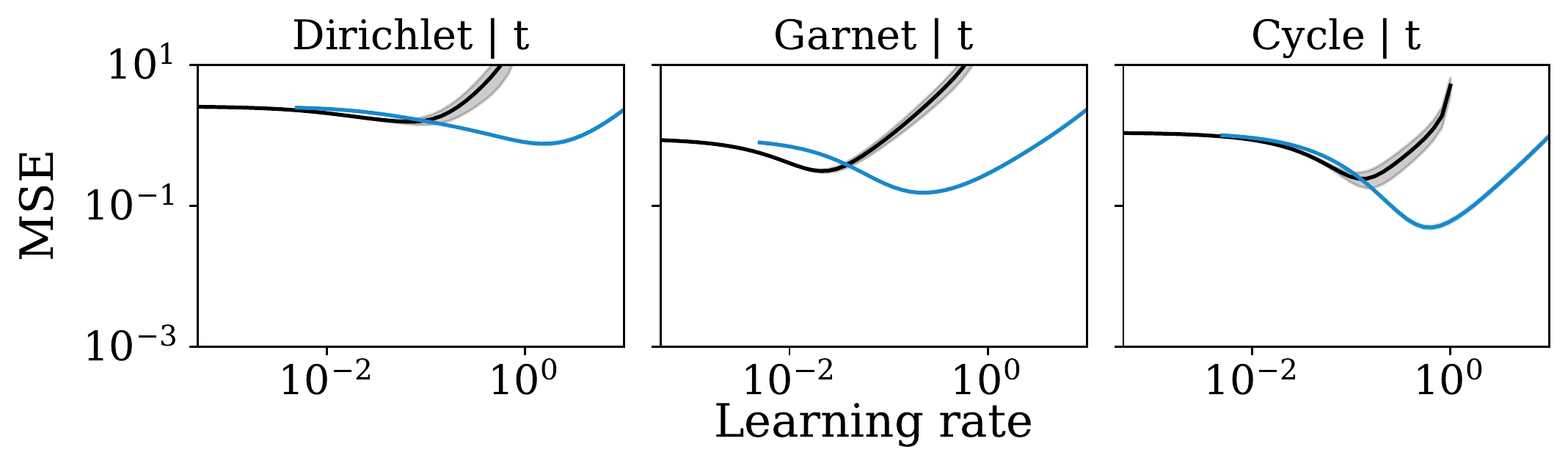}
    
    \caption{Mean-square error against learning rate for both TD and QTD run for 1,000 updates in environments with $t_2$-distributed rewards.}
    \label{fig:t-lr}
\end{figure}

\subsection{Varying numbers of quantiles}\label{sec:app:main-sweep-m}

We present results for the relative improvement of QTD(1) and QTD(16) over TD on the main suite of environments in Figures~\ref{fig:m1-sweep} and \ref{fig:main-sweep-m16}, respectively.
Unlike QTD(128), which is superior to TD on all stochastic environments in the suite, QTD(1) is outperformed by TD on several stochastic environments, showing that performance can be degraded when not using a sufficient number of quantiles.
Interestingly, there are also environments in which the performance of QTD(1) is not degraded relative to QTD(128). We note in particular that the performance of QTD(1) on the cycle Gaussian environment in fact improves over that of QTD(128). In fact, in environments with deterministic transitions and certain symmetric reward distributions, there is no fixed-point error in the value predictions of QTD(1); this is intuitively driven by the agreement of the median and the mean of the distributional Bellman targets in this environment, in contrast to e.g.\ the example in Figure~\ref{fig:low-m}. We note also the oscillations that appear in certain environments in Figure~\ref{fig:m1-sweep}; these artefacts are due to the discrete grid of learning rates that are swept over. Unlike QTD(1), QTD(16) dominates TD in the stochastic environments within the suite, and the performance is comparable with that of QTD(128), indicating the diminishing returns of increasing the number of quantiles to be estimated beyond a certain range.

\begin{figure}
    \centering
    \includegraphics[keepaspectratio,width=.49\textwidth,valign=c]{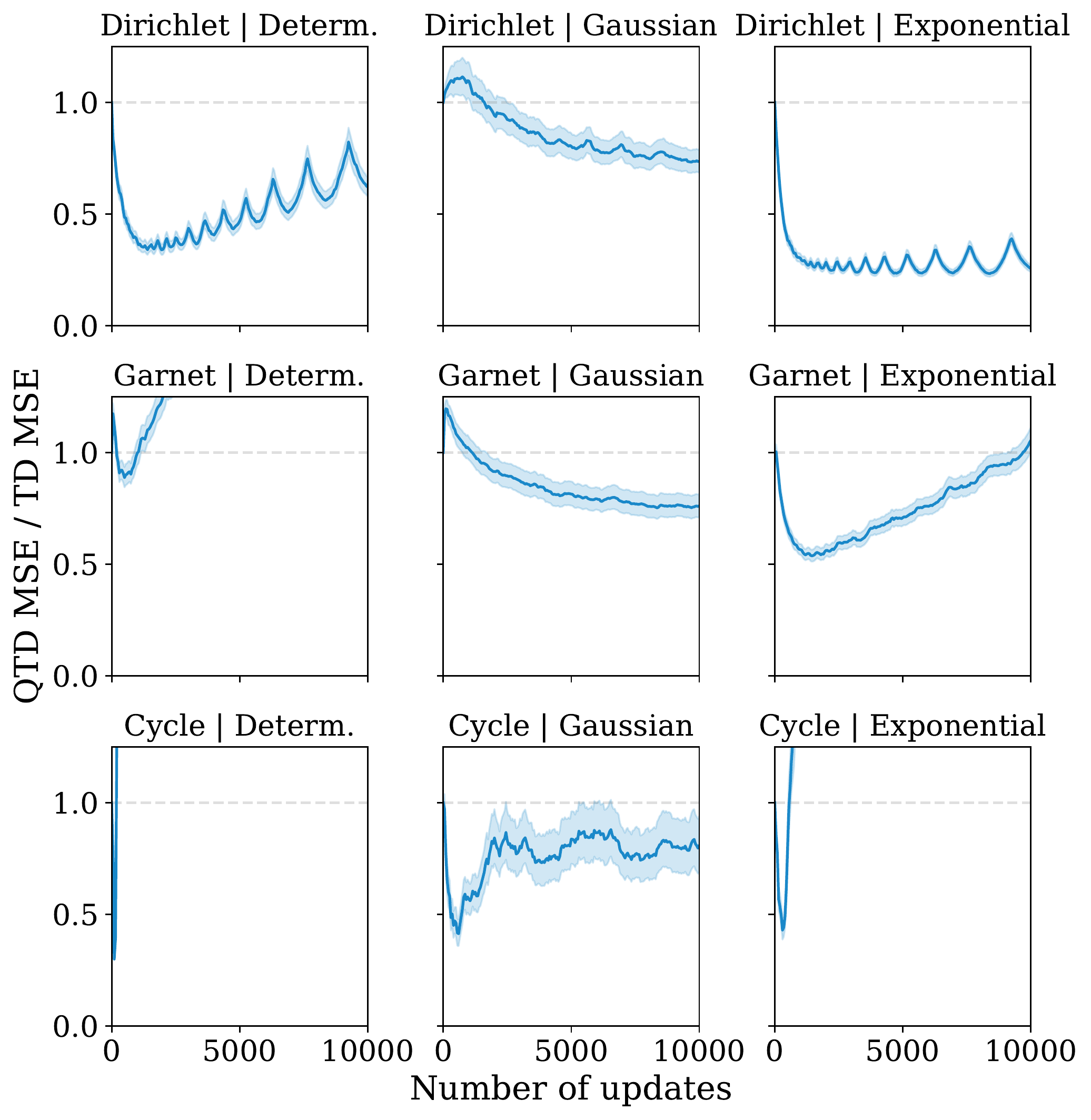}
    
    \caption{Relative improvement of QTD(1) over TD in mean-squared error against number of updates.}
    \label{fig:m1-sweep}
\end{figure}

\begin{figure}[h]
    \centering
    \includegraphics[keepaspectratio,width=.49\textwidth,valign=c]{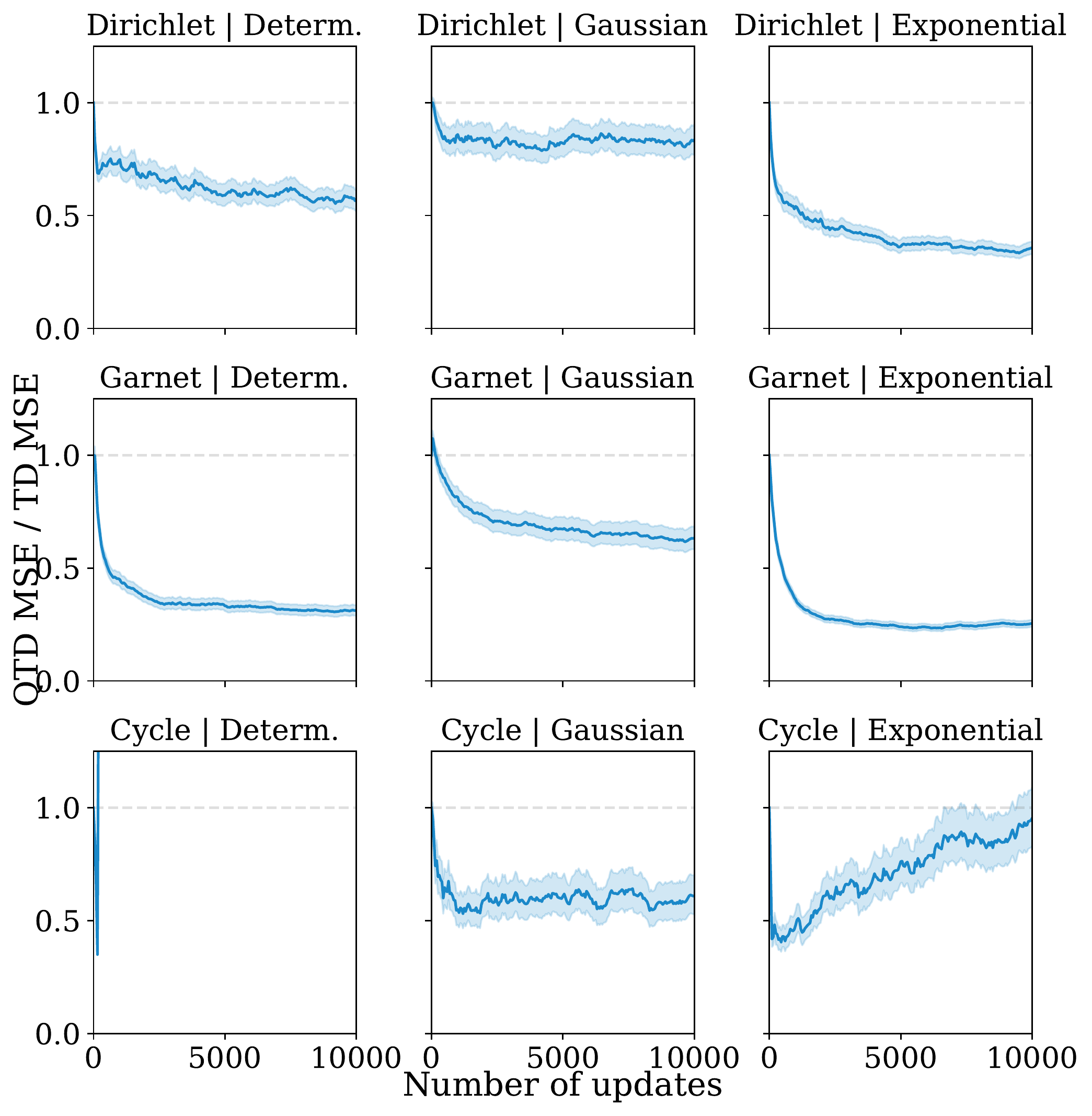}
    \caption{Relative improvement of QTD(16) over TD in mean-squared error against number of updates.}
    \label{fig:main-sweep-m16}
\end{figure}

\end{appendix}

\end{document}